\documentclass[]{dmir}
\usepackage{graphicx} 
\usepackage{xcolor} 
 
\usepackage{natbib}
\usepackage{amsthm}
\usepackage{amsmath}
\usepackage{amssymb}
\usepackage{subcaption}
\usepackage{booktabs}
\usepackage{multirow}
\usepackage{float}
\usepackage{hyperref} 
\newtheorem{theorem}{Theorem}
\newtheorem{corollary}{Corollary}

\title{CDFM: Towards a General-Purpose Causal Discovery Foundation Model}
\abstract{
Causal discovery, the process of recovering underlying causal structures from observational data, is a fundamental pursuit across scientific disciplines. Over the past decades, numerous algorithms have been developed to tackle this challenge through workflows tailored to the specific causal mechanisms underlying each type of dataset, demonstrating effectiveness across a wide range of applications. However, as the volume and heterogeneity of real-world data continue to grow, this dataset-specific approach inevitably leads to a fragmented, test-driven paradigm that struggles to scale to the demands of modern scientific discovery. To address this, we formulate the Causal Discovery Foundation Model (CDFM) as a unified, general-purpose framework for zero-shot structural inference. To ensure reliable generalization across unknown domains, we first investigate the theoretical boundaries of causal identifiability, revealing the indispensable role of causal prior mechanisms in this process. Building on these insights, we formulate a principled variational framework that treats unknown causal mechanisms as latent variables and mathematically decomposes the intractable marginal likelihood into distinct, tractable learning modules. The variational decomposition provides a conceptual design principle for the architecture design of CDFM, while comprehensive causal knowledge guides the large-scale synthesis of our pretraining data. By pretraining on a massive, highly diverse space of synthetic structural causal models, CDFM successfully internalizes complex statistical asymmetries. Extensive experiments demonstrate that CDFM consistently outperforms traditional algorithms, driving a paradigm shift toward a general-purpose causal discovery foundation model. 
}

\author[1]{Jie~Qiao}
\author[1*]{Ruichu~Cai} 
\author[1]{Zijian~Li}
\author[1]{Weilin~Chen}
\author[1]{Pengfei~Hua}
\author[1]{Boyan~Xu}
\author[2]{Zhengming~Chen}
\author[2]{Zhifeng~Hao}
\author[3]{Peng~Cui}

\affiliation[1]{School of Computer Science, Guangdong University of Technology, Guangzhou, China}
\affiliation[2]{College of Mathematics and Computer, Shantou University, Shantou, China}
\affiliation[3]{Department of Computer Science and Technology, Tsinghua University, Beijing, China}

\authorcustomemails{\{qiaojie.chn,cuiruichu,leizigin,chenweilin.chn,hua13662626240,hpakyim,chenzhengming1103\}@gmail.com, haozhifeng@stu.edu.cn, cuip@tsinghua.edu.cn}

\correspondence{Ruichu Cai (cairuichu@gmail.com)}
\code{The code is available at \url{https://github.com/DMIRLAB-Group/CDFM}.}

\begin{document}

\maketitle

\section{Introduction}

Causal discovery, the process of recovering underlying causal structures from observational data, is a fundamental pursuit across scientific disciplines, with broad applications in fields such as economics~\citep{ghysels2016testing}, biology~\citep{sachs2005causal}, and social science~\citep{cai2016understanding}.
Over the past decades, numerous causal discovery algorithms have been developed, demonstrating effectiveness across a wide variety of datasets~\citep{gamella2025causal, mooij2016distinguishing, runge2020causality}. 

Despite these advancements, the field remains fundamentally bottlenecked by the traditional causal discovery workflow, which is highly fragmented, test-driven, and complex. As shown in Figure \ref{fig:paradigm_shift}, existing algorithms are meticulously designed to operate under specific causal mechanisms (e.g., linear non-Gaussianity \citep{shimizuLinearNonGaussianAcyclic2006} or additive noise \citep{hoyerNonlinearCausalDiscovery2009}, etc.). Consequently, these algorithms excel only when their underlying assumptions strictly hold. This creates isolated approaches that fail to generalize across heterogeneous real-world data, forcing practitioners into an exhaustive loop of manually re-evaluating algorithms and statistical tests.

As the volume and heterogeneity of data continue to grow, this manual, hypothesis-matching paradigm becomes increasingly untenable; it cannot scale to the demands of modern, automated scientific discovery. What is needed is not another assumption-specific algorithm, but a unified model that can automatically accommodate—or even infer—the appropriate causal structures and mechanisms directly from data. Hence, we are led to the central question: \textit{Is it possible to develop a foundation model that transcends rigid assumptions to enable a general-purpose approach to causal discovery?}

Recent foundation models for tabular data provide evidence that this strategy is plausible. Models such as TabPFN~\citep{hollmann_accurate_2025}, LimiX~\citep{zhang2025limix}, and TabICL~\citep{qu2025tabicl} learn reusable inductive biases from large synthetic or heterogeneous pretraining distributions. Related work has extended this idea to causal effect estimation~\citep{balazadeh2026causalpfn,robertson2026dopfn,ma2026foundation} and amortized causal discovery~\citep{lorchAmortizedInferenceCausal2022,swelam2025tabpfn}. However, a general formulation connecting causal identifiability, mechanism uncertainty, architecture design, and pretraining coverage remains underdeveloped.

We introduce the \textbf{C}ausal \textbf{D}iscovery \textbf{F}oundation \textbf{M}odel (CDFM), a framework for direct structural inference across heterogeneous and initially unknown causal mechanisms. CDFM is pretrained on diverse synthetic data sampled from a broad space of identifiable causal assumptions, incorporating rich causal knowledge. The resulting model emerges as a general-purpose causal discovery foundation model that is capable of zero-shot structural inference across diverse, unknown mechanisms.

\begin{figure}
    \centering
    \includegraphics[width=\textwidth]{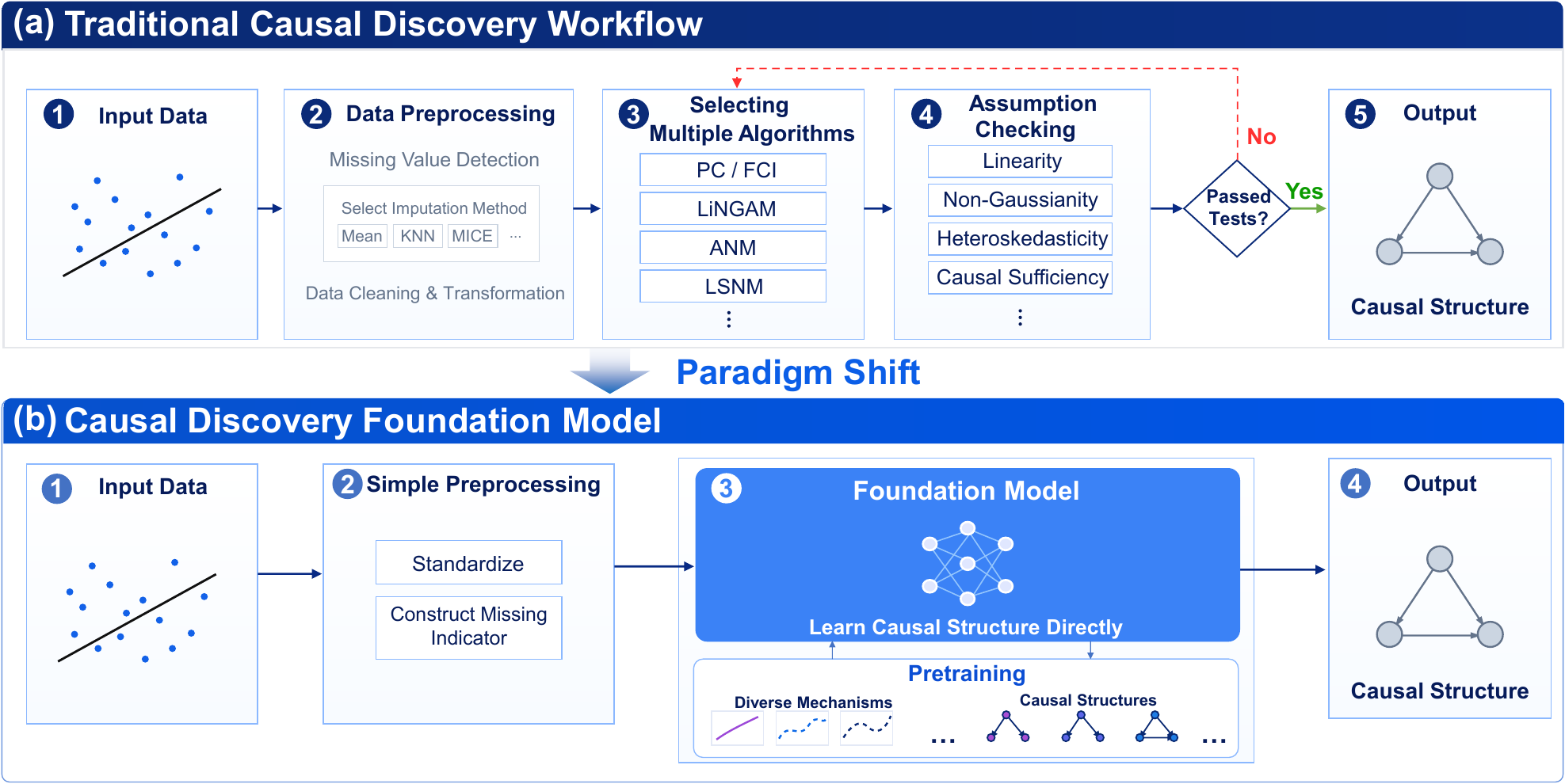}
    \caption{Paradigm shift from traditional causal discovery to CDFM. Instead of selecting algorithms and verifying assumptions for each dataset, CDFM performs lightweight preprocessing and directly infers causal structures using a pretrained foundation model. The model is pretrained on diverse synthetic causal mechanisms and structures, enabling generalized causal reasoning across heterogeneous data distributions.}
    \label{fig:paradigm_shift}
\end{figure}

The main contributions of this work are summarized as follows:
\begin{itemize}
    \item We formulate a causal discovery foundation model, a paradigm that is specifically tailored for general-purpose causal discovery, shifting the paradigm from specialized algorithm selection to a unified inference approach. To mathematically ground this, we unify foundation models and structural identifiability through a principled variational framework, treating unknown mechanisms as latent variables to translate the marginalization problem into a tractable learning objective.
    \item We design CDFM that aligns mechanism encoding, structure-aware reconstruction, and graph decoding with the theoretical decomposition, allowing the model to robustly infer causal structures directly from observational data.
    \item We introduce a comprehensive, large-scale pretraining driven by diverse causal knowledge. By synthesizing a massive hypothesis space of identifiable causal mechanisms and structures, we enable the model to internalize complex statistical asymmetries and effectively generalize to unknown generating processes, and extensive experiments demonstrate that CDFM achieves zero-shot generalization across diverse data distributions and real-world datasets.
    \end{itemize}

\section{Related Work}
The methodology of causal discovery from observational data has evolved significantly in recent years. In this section, we review the literature, tracing the progression from theoretically constrained algorithms to amortized inference and the generalized foundation models.

\paragraph{Constrained causal discovery.} The traditional approach to causal discovery relies on algorithms derived under specific constraints about the data-generating process. Typical independence constraint-based methods, such as PC and FCI \citep{spirtesCausationPredictionSearch2000}, exploit conditional independencies under the causal Markov and faithfulness assumptions, typically identifying structures only up to a Markov equivalence class (MEC). To resolve MEC ambiguities, additional Functional Causal Models (FCMs) are included for imposing strict structural constraints. Linear non-Gaussian models (LiNGAM) \citep{shimizuLinearNonGaussianAcyclic2006, shimizuDirectLINGAMDirectMethod2011} break symmetry by assuming linear equations with non-Gaussian noise. Additive noise models (ANM) \citep{hoyerNonlinearCausalDiscovery2009} and post-nonlinear models \citep{zhangIdentifiabilityPostNonlinear2009} extend this identifiability to nonlinear regimes. This family is further enriched by causal additive models \citep{buhlmannCAMCausalAdditive2014} and causal location-scale noise models \citep{immer2023identifiability, yinEffectiveCausalDiscovery2024}, which leverage dependent noise variance as an additional source of causal asymmetry. For discrete variables, a series of works have addressed ordinal, categorical, and count data \cite{petersIdentifyingCauseEffectDiscrete2010,niOrdinalCausalDiscovery2022,yaoCausalDiscoveryMixed2025,caiCausalDiscoveryDiscrete2018,qiao2026identifiability,qiao2024causal,xiang2024identifiability,qiaoStructuralHawkesProcesses2023}. 
Later advancements transitioned from combinatorial search to continuous optimization (e.g., NOTEARS \citep{zhengDAGsNOTEARSContinuous2018}), utilizing smooth algebraic characterizations of acyclicity. While these methods provide rigorous identifiability guarantees within their respective domains, they require a priori knowledge of the underlying mechanisms, lacking the generalizability to the heterogeneous real-world data where the true data-generating process is unknown.

\paragraph{Amortized causal discovery.}
To overcome the bottleneck of solving dataset-specific optimization problems, amortized inference reframes causal discovery as a supervised learning task. Methods like AVICI \citep{lorchAmortizedInferenceCausal2022} and CSIvA \citep{keLearningInduceCausal2023} train neural networks to infer causal graphs directly from observational and interventional samples. This learning-based framework has been subsequently expanded across various dimensions, including applications to time-series data \citep{pmlr-v177-lowe22a}, active causal induction using deep reinforcement learning \citep{annadani2024caasl}, and scalable conditional independence testing \citep{pmlr-v275-leiner25a}. Further methodological extensions \citep{thompson2026arrow, wu2025sea, yin2025learning, peng2026causcale} have demonstrated the empirical feasibility of large-scale causal discovery.
While these amortized models successfully internalize useful inductive biases from simulated data via empirical heuristics, a systematic theoretical framework for addressing structural identifiability across highly heterogeneous and unknown mechanisms remains unresolved.

\paragraph{Foundation models for tabular data.}
Large-scale pretraining has recently demonstrated remarkable in-context generalization on tabular and structured data tasks. Models such as LimiX \citep{zhang2025limix}, TabPFN \citep{hollmann_accurate_2025}, and TabICL \citep{qu2025tabicl} have shown that foundation architectures can effectively model complex, heterogeneous datasets. Within the causal inference domain, efforts like CausalPFN \citep{balazadeh2026causalpfn} and Do-PFN \citep{robertson2026dopfn} have successfully adapted these architectures for causal effect estimation. These advancements illustrate the profound potential of foundation models to generalize across disparate data distributions.

\paragraph{Toward causal discovery foundation models.}
The evolution of causal discovery reveals a fundamental paradigm shift: from isolated, assumption-bound algorithms to amortized inference, and ultimately, causal discovery foundation models. In this work, we systematically study the causal discovery foundation model, establishing a principled methodology for general-purpose causal discovery.

\section{Background and Preliminaries}
\label{sec: background preliminaries}

We represent the data-generating process using a Structural Causal Model (SCM). Let $\mathbf{X} = \{X_1, X_2, \dots, X_D\}$ be observed variables, where $D$ denotes the number of variables, and $\mathbf{L}$ be unobserved latent confounders. The system is governed by structural assignments:
\begin{equation}
    V_i := f_i(\mathbf{PA}_i, N_i), \quad \text{for } V_i \in \mathbf{V} = \mathbf{X} \cup \mathbf{L},
\end{equation}
where $\mathbf{PA}_i$ are direct causal parents, $f_i$ is a deterministic mechanism, and $N_i$ is exogenous noise. By marginalizing out $\mathbf{L}$, we represent the observed structure as an adjacency matrix $G \in \{0,1\}^{D \times D}$ encoding the causal structure over the observed variables after marginalizing out $\mathbf{L}$. In addition, to capture the effects of latent confounding, we augment the graph with bidirected edges, signifying the existence of at least one unobserved common cause $L \in \mathbf{L}$ that directly influences both variables. Note that if both $X_i \rightarrow X_j$ and latent confounding exist, we encode only the directed edge, and we restrict the target graph class to bow-free acyclic directed mixed graphs.

Given this formulation, our primary goal is to robustly infer the underlying causal graph $G$ from purely observational data. Crucially, unlike traditional approaches that rely on manually specifying a restrictive causal prior mechanism $M$, we aim to uncover these causal relationships across a massive space of unknown and heterogeneous mechanisms $\mathcal{M}$.

\section{Paradigm Shift: From Identifiability Theory to Foundation Models}\label{sec:identifiability}

As discussed earlier by Pearl \cite{pearl2010introduction}, behind any causal conclusion, there must be some causal assumption, untested in observational studies. The role of a foundation model is therefore not to remove assumptions, but to learn and adopt over a broad spectrum of assumption classes. To develop principled guidelines for the design of CDFM, enabling it to generalize across diverse causal mechanisms, we therefore first investigate the theoretical boundaries of causal identifiability and guide the architecture design of CDFM.

In causal discovery, underlying assumptions are typically encapsulated implicitly within structural causal models. However, as established by \cite{peters2017elements}, an SCM with independent noise alone is insufficient to identify the causal direction.

\begin{theorem}[Impossibility of Learning Causal Direction without Constraint] \label{thm:impossibility}
Let $X$ and $Y$ be two random variables. For any joint distribution $P_{X,Y}$, valid SCMs exist for both causal orderings simultaneously. That is, there exist measurable functions $f_X, f_Y$ and real-valued noise variables $N_X, N_Y$ such that:
\begin{equation}
\begin{aligned}
    Y &= f_Y(X, N_Y), \quad X \perp\!\!\!\perp N_Y, \\
    X &= f_X(Y, N_X), \quad Y \perp\!\!\!\perp N_X.
\end{aligned}
\end{equation}
\end{theorem}
Crucially, because this holds for any joint distribution, the variables are perfectly symmetric. Thus, identifying the true causal graph is mathematically impossible without imposing prior assumptions on the underlying structural equations. Traditional algorithms resolve this by manually asserting a specific, known causal prior mechanism (e.g., linear non-Gaussianity \citep{shimizuDirectLINGAMDirectMethod2011}, or additive noise model \citep{hoyerNonlinearCausalDiscovery2009}). Given that the true mechanism $M^*$ is known, identifiability transforms into a structural model selection problem. The optimal graph minimizes the Kullback-Leibler (KL) divergence between the true distribution and the model distribution.
\begin{theorem}[Identifiability under Known Mechanism]
\label{thm:kl_known_mechanism}
Let the data $\mathbf{X}$ be generated by a true graph $G^*$ and a known causal mechanism $M^*$. For any alternative graph $G'$, the difference in expected log-likelihood under the true distribution $\mathcal{L}(G^*):= \mathbb{E}_{\mathbf{X}\sim P(\cdot\mid G^*,M^*)}\left[\log P(\mathbf{X}\mid G^*,M^*)\right]$ is bounded by the KL divergence:
\begin{equation}
\mathcal{L}(G^*) - \mathcal{L}(G') = D_{\mathrm{KL}}\big[ P(\mathbf{X} | G^*, M^*) \parallel P(\mathbf{X} | G', M^*) \big] \ge 0.
\end{equation}
If $M^*$ belongs to a strictly identifiable class, the inequality is strict for every $G'\neq G^*$, and equality holds if and only if $G'=G^*$.
\end{theorem}

Unlike traditional algorithms, a foundation model drives a paradigm shift in which it does not assume a single, known mechanism $M^*$. Instead, it aims to map an observational dataset $\mathbf{X} \in \mathbb{R}^{N \times D}$ directly to the causal graph $G$ across a vast hypothesis space of diverse, unknown mechanisms $\mathcal{M}$. This shift means the mapping inevitably amounts to a Bayesian model selection problem, where the optimal likelihood requires marginalizing over the entire mechanism space $\mathcal{M}$:
\begin{equation}
\label{eq:marginal_likelihood}
    \mathbb{E}_{\mathbf{X}}\left[ \log P(\mathbf{X},G) \right] = \mathbb{E}_{\mathbf{X}}\left[ \log \int_{\mathcal{M}} P(\mathbf{X} | M, G) P(G | M)P(M) dM \right].
\end{equation}
One may still identify the causal structure using the exact marginal likelihood as long as the space of causal prior mechanisms $\mathcal{M}$ satisfies a strict identifiability property. We say the space $\mathcal{M}$ is strictly identifiable if no alternative structure and mechanism can perfectly mimic the true generating distribution. Formally, $P(\mathbf{X}|M^*,G^*) \neq P(\mathbf{X}|M_{G'},G')$ holds for all $G' \neq G^*$, where $M_{G'} = \arg\max_{M \in \mathcal{M}} \mathbb{E}_{\mathbf{X}}[\log P(\mathbf{X}|M,G')]$ represents the optimal mechanism choice under $G'$. If the space possesses this intrinsic mutual exclusivity, the exact marginal likelihood uniquely identifies the true causal graph.
\begin{theorem}[Identifiability by Marginal Likelihood]
\label{thm:marginal_identifiability}
Let a dataset $\mathbf{X}^{(N)}$ of $N$ i.i.d. samples be generated by the true graph $G^*$ and true mechanism $M^* \in \mathcal{M}$. If the space of causal prior mechanisms $\mathcal{M}$ satisfies the strict space identifiability property, then the exact marginal likelihood uniquely identifies the true causal graph almost surely:
\begin{equation}
\lim_{N \to \infty} \frac{1}{N} \log P(\mathbf{X}^{(N)} | G^*) > \lim_{N \to \infty} \frac{1}{N} \log P(\mathbf{X}^{(N)} | G'), \quad \text{a.s.}, \quad \forall G' \neq G^*.
\end{equation}
\end{theorem}
While Theorem \ref{thm:marginal_identifiability} guarantees that the true graph can be identified in principle, directly computing the exact integral in Eq. \eqref{eq:marginal_likelihood} is intractable. 

To bridge this gap, we formulate the learning process of CDFM through variational inference. Under uninformative uniform priors of $P(G)$ and $P(M)$, maximizing the marginal likelihood is mathematically equivalent to maximizing the joint likelihood. We introduce a parameterized variational distribution $Q(M \mid \mathbf{X})$, which is designed to internalize the mechanism and structural information within the architecture. This yields a tractable Evidence Lower Bound (ELBO) over the joint likelihood (see Appendix \ref{app:elbo_derivation} for derivation):
\begin{equation}
\label{eq:elbo}
\log P(G, \mathbf{X}) \ge \mathbb{E}_{Q(M | \mathbf{X})} \big[ \log P(\mathbf{X} | G, M) + \log P(G | M) \big] - D_{\mathrm{KL}}\big( Q(M | \mathbf{X}) \parallel P(M) \big),
\end{equation}
where the final KL divergence term serves as a regularizer over the mechanism space, which acts as a constant shift under a standard uninformative uniform prior.
This variational formulation explicitly decomposes the learning process of CDFM by decomposing it into three explicit components that govern the design of CDFM:
\begin{enumerate}
    \item \textbf{Mechanism Inference $Q(M | \mathbf{X})$}: The model extracts statistical footprints (e.g., noise properties, non-linearities) to infer a latent encoding of the active mechanism and the structural information.
    \item \textbf{Data Reconstruction $\log P(\mathbf{X} | G, M)$}: Theoretical identifiability requires the inferred mechanism and graph to perfectly explain the data. This provides a self-supervised gradient signal.
    \item \textbf{Graph Inference $\log P(G | M)$}: The model decodes the structural graph conditioned on the neurally encoded mechanism and its structural information.
\end{enumerate}

By this, with sufficient capacity and massive pretraining, the variational gap ($D_{\mathrm{KL}}$ between the true and approximate posterior) vanishes at the global optimum. Consequently, optimizing this ELBO implicitly maximizes the exact marginal likelihood. This guarantees that CDFM inherently respects the mathematical boundaries of causal identifiability:
\begin{corollary}[Identifiability of CDFM]\label{cor:elbo_identifiability}
Suppose the foundation model has sufficient capacity to reach the optimum variational ELBO. Let $\mathbf{X}^{(N)}$ be any test dataset with unknown causal graph $G^*$. In the inference phase, the model predicts the graph by maximizing the expected log-likelihood of the structural decoder:
\begin{equation}
    \hat{G} = \arg\max_{G \in \mathcal{G}} \mathbb{E}_{Q(M \mid \mathbf{X}^{(N)})} \big[ \log P(G \mid M) \big].
\end{equation}
Under the strict identifiability of the mechanism space $\mathcal{M}$ and a uniform prior $P(G)$, this structural inference uniquely recovers the true causal graph almost surely, i.e., $\hat{G} = G^*$ as $N \to \infty$.
\end{corollary}

\begin{figure}[t]
    \centering
    \includegraphics[width = 0.5\textwidth]{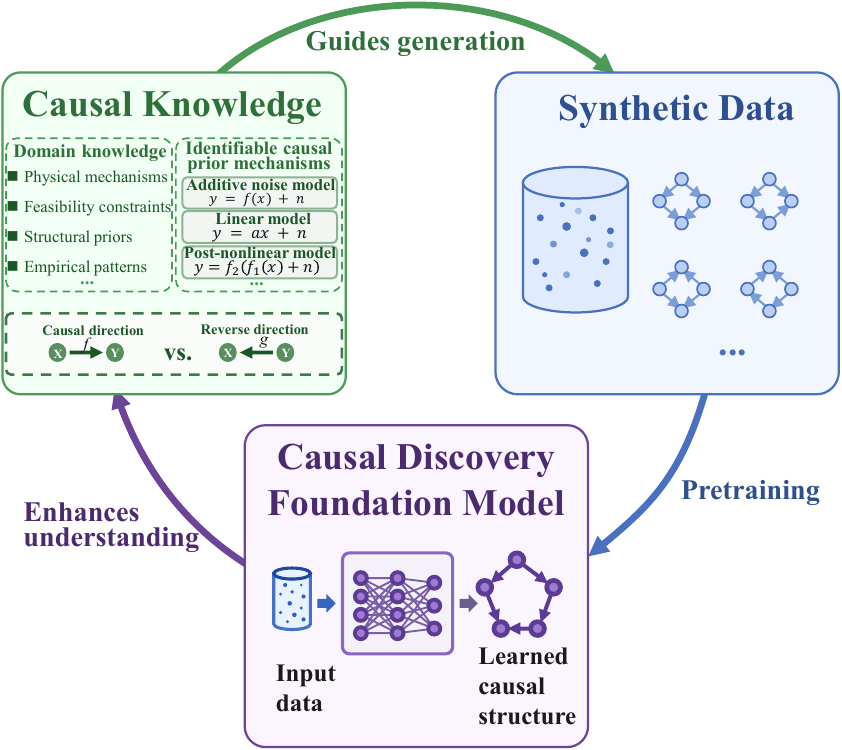}
    \caption{The foundation model paradigm for causal discovery. Causal knowledge guides the generation of diverse synthetic data. Pre-training on this data allows the foundation model to internalize causal asymmetries. Zero-shot inference on real-world data enhances our understanding of unknown mechanisms, which in turn helps us refine and expand the causal knowledge, ultimately driving towards a universally generalizable model.}
    \label{fig:framework}
\end{figure}

By establishing that our architecture is consistent with constraints via the variational framework, we formalize a fundamental paradigm shift in causal discovery, as illustrated in Figure \ref{fig:framework}. In traditional frameworks, theoretical identifiability merely served as a prerequisite for the specified causal mechanisms. In CDFM, however, Corollary \ref{cor:elbo_identifiability} shifts the core bottleneck from manual algorithm selection to the scalable generation of diverse, identifiable synthetic data. 

This forms an iterative, self-reinforcing closed loop. Existing causal knowledge guides the generation of diverse synthetic datasets. Pre-training on this vast space allows the model's parameterized distributions to internalize complex causal asymmetries. Subsequently, zero-shot inference on real-world data evaluates the model's structural predictions. Insights gained from its performance reveal previously unknown causal mechanisms. This feedback empowers us to refine our theoretical understanding of identifiability, expand the underlying causal knowledge base, and generate even richer synthetic data, ultimately driving the foundation model toward a universally generalizable causal discovery engine.

\begin{figure}[t]
    \centering
    \includegraphics[width=\textwidth]{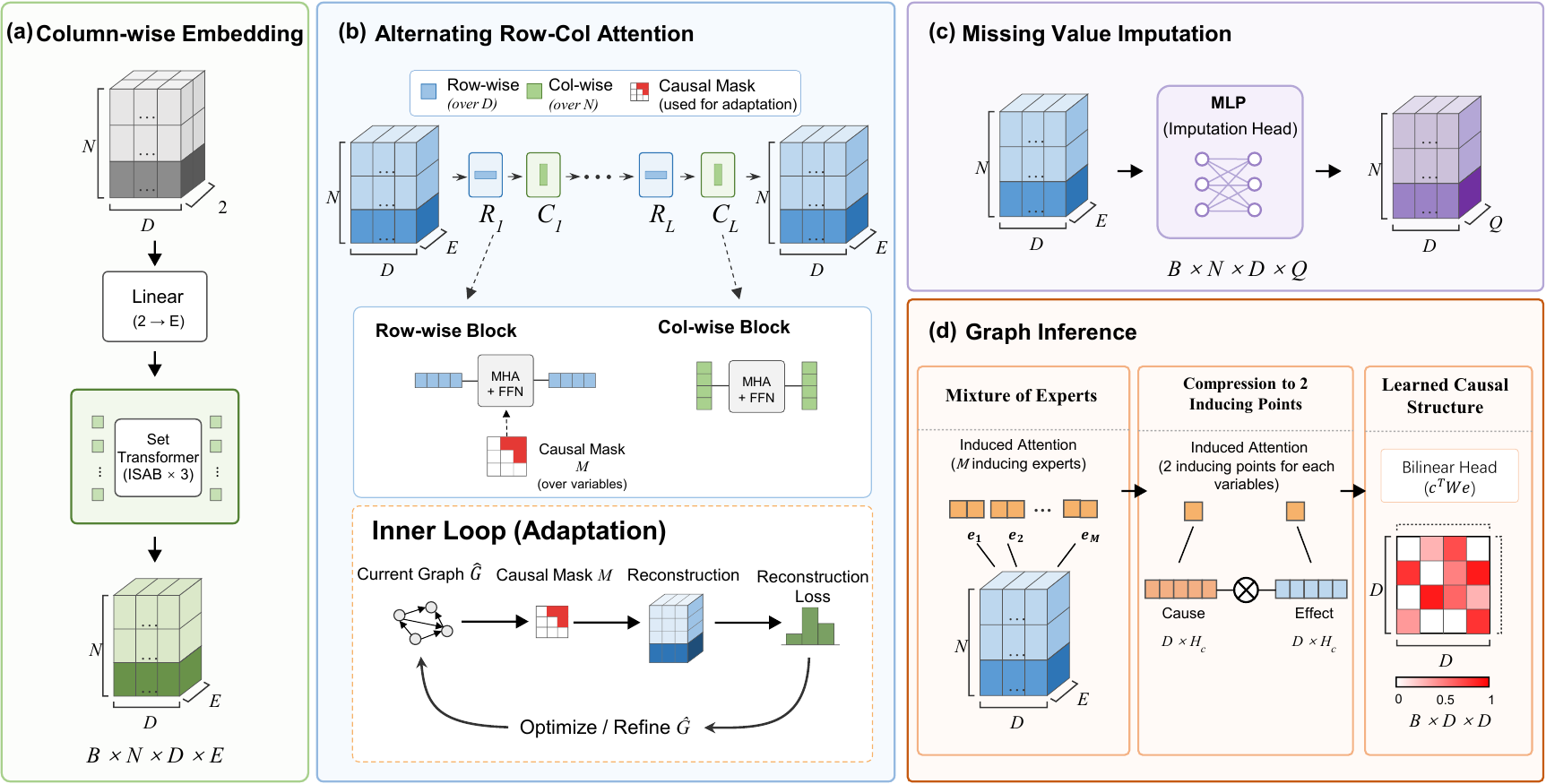}
    \caption{
    The overall architecture of CDFM.
    }
    \label{fig:gcd_fm_architecture}
\end{figure}

\section{Causal Discovery Foundation Model}

CDFM is designed to instantiate the intractable marginalization over the mechanism space $\mathcal{M}$ through a deep neural architecture. It takes an observation matrix $\mathbf{X} \in \mathbb{R}^{N \times D}$ (consisting of $N$ observations over $D$ variables, with potential missing entries) and infers the underlying causal graph $G \in \{0,1\}^{D \times D}$. As formalized in Section~\ref{sec: background preliminaries}, the graph would be an adjacency matrix augmented with bidirected edges when latent confounding is present.

The architectural design of CDFM explicitly mirrors the three components of the variational ELBO: mechanism inference (Section \ref{sec:Mechanism Inference}), data reconstruction (Section \ref{sec:Data Reconstruction}), and graph inference (Section \ref{sec:Graph Inference}).

The overview of the architecture of CDFM, as illustrated in Figure~\ref{fig:gcd_fm_architecture}, mainly consists of three distinct components:
\begin{enumerate}
    \item \textbf{Column-Wise Embedding and Alternating Attention (Mechanism Inference):} Each scalar is first embedded into an $E$-dimensional vector via a 2-channel, distribution-aware Set Transformer to capture marginal footprints and missing-value patterns. The embedded tensor is then processed by a stack of $2L$ attention blocks that alternate between the sample and variable axes, integrating joint dependencies to parameterize the latent mechanism distribution $Q(M|\mathbf{X})$.
    \item \textbf{Missing Value Imputation (Data Reconstruction):} A residual block predicts $Q$ conditional quantile values for each masked position, serving as a self-supervised task to satisfy the data reconstruction term $P(\mathbf{X}|G,M)$.
    \item \textbf{Extraction and Graph Inference:} Induced self-attention extracts cause and effect summary vectors, from which a Mixture-of-Experts bilinear head produces the final edge logits, which fulfill the structural decoding term $P(G|M)$.
\end{enumerate}

\subsection{Mechanism Inference: Column-Wise Embedding and Alternating Attention}\label{sec:Mechanism Inference}

To build the variational distribution $Q(M|\mathbf{X})$, CDFM first extracts high-order statistics of each feature from the data by transforming each scalar cell into a dense vector representation. To achieve this while explicitly supporting incomplete real-world datasets and paving the way for the subsequent self-supervised data reconstruction task, CDFM begins with a column-wise embedding module based on a two-channel input encoding. Specifically, each observed scalar value $x_{ij}$ is paired with a binary missingness indicator $r_{ij} \in \{0, 1\}$, forming a two-dimensional vector $[x_{ij}, r_{ij}]$. For a missing entry, $x_{ij}$ is replaced by zero after column-wise standardization, while $r_{ij}=1$ explicitly marks its missingness. A linear projection maps this pair into the initial embedding space: 
\begin{equation}
    u_{ij} = \mathrm{Linear}_{\mathrm{in}}([x_{ij},\, r_{ij}]) \quad \in \mathbb{R}^{E}.
\end{equation}
This missingness channel allows the model to distinguish the missing entries, aiding downstream imputation by marking which positions require reconstruction.

Subsequently, inspired by \citep{qu2025tabicl}, to encompass the high-order statistics of marginal distributions, we employ a shared Set Transformer~\citep{lee2019set} that treats each column as an \textit{unordered set of cell values}. For each column $c_j = \{x_{1j}, \dots, x_{Nj}\}$, the projected embeddings are processed by a set transformer composed of $L_{\mathrm{col}} = 3$ bidirectional induced self-attention blocks (ISAB)~\citep{lee2019set}. Each ISAB block operates as follows:
\begin{equation}
    \begin{aligned}
        H &= \mathrm{MAB}_1(V_I,\, U,\, U) \quad \in \mathbb{R}^{K \times E}, \\
        V &= \mathrm{MAB}_2(U,\, H,\, H) \quad \in \mathbb{R}^{N \times E},
    \end{aligned}
\end{equation}
where MAB denotes a multi-head attention block, $V_I \in \mathbb{R}^{K \times E}$ are $K$ learnable inducing vectors, and $U$ refers to the embedded samples within the column. After the ISAB blocks, two independent linear projections generate per-sample affine transformation parameters:
\begin{equation}
   W_j=\mathrm{Linear}_W(V),
\qquad
B_j=\mathrm{Linear}_B(V),
\end{equation}
and the final column embedding is obtained via an element-wise affine transform applied only to the value channel:
\begin{equation}
    e_{ij} = W_{ij} \odot x_{ij} + B_{ij} \quad \in \mathbb{R}^{E}.
\end{equation}
This affine gating mechanism inspects the empirical distribution of the entire column and generates per-cell weights and biases that encode each value's position relative to the column's distribution. Empirically, this causes columns with similar distributional properties to cluster in embedding space, effectively capturing the structural footprints of the mechanisms.

Building upon these marginal footprints extracted from individual columns, CDFM proceeds to capture the joint dependencies across variables to fully parameterize the variational distribution $Q(M|\mathbf{X})$. This is achieved through a sequence of alternating row- and column-wise attention operations that progressively integrate information across both axes of the data, as illustrated in Figure \ref{fig:gcd_fm_architecture}. Each attention step employs a multi-head attention block followed by a feedforward network with GELU activation, synthesizing complex variable-sample interactions into a comprehensive latent mechanism state.

\subsection{Data Reconstruction: Causal Masking and Imputation}\label{sec:Data Reconstruction}

To instantiate the data reconstruction term $\log P(\mathbf{X} \mid G, M)$, CDFM employs a self-supervised missing value imputation task guided by a soft structural masking mechanism.

During the forward pass, even-indexed attention blocks incorporate a soft causal mask $\mathbf{S} \in \mathbb{R}^{D \times D}$ that dictates the information flow between variables. Rather than applying a rigid, non-differentiable penalty, CDFM employs a continuous relaxation parameterized by the continuous edge weights $W$ inferred from the current graph hypothesis. When variable $i$ attends to variable $j$ ($i \neq j$), the attention logit is penalized based on the strength of the directed edge $j \rightarrow i$:
\begin{equation}
    S_{ij} = -\beta \cdot \sigma\left( - \frac{W_{ji} - \tau}{T} \right)
\end{equation}
where $\sigma$ is the sigmoid function, $\tau$ is a decision threshold, $T$ is a temperature parameter controlling the steepness of the transition, and $\beta$ is a scaling factor determining the maximum penalty. For self-attention, the diagonal remains unpenalized ($S_{ii} = 0$).

This soft inductive bias ensures that if $j$ is not a likely cause of $i$ ($W_{ji} \ll \tau$), the attention mechanism heavily restricts information flow from $j$ to $i$. Conversely, strong causal edges allow unhindered attention. This formulation encourages the model to evaluate the target variable as a function primarily of its direct causes.

Following the final alternating block, the model gathers the contextualized embeddings $z^{(2L)}$ at positions originally marked as missing ($r_{ij} = 1$) and processes them through a residual block to predict quantile levels $\hat{y}_{ij} \in \mathbb{R}^{Q}$. By evaluating the discrepancy between these predictions and the masked true values, the imputation task provides a continuous gradient signal for structure learning.

Crucially, because the mask $\mathbf{S}$ is soft and differentiable with respect to the edge weights $W$, it seamlessly facilitates an unsupervised inner loop during inference. In this adaptation phase, the model iteratively updates the continuous causal graph hypothesis by minimizing the imputation error. This inner loop allows CDFM to refine its structure predictions on novel datasets without relying on ground-truth causal labels, aligning with the principle that an accurate causal graph should yield minimal observational reconstruction error.

\subsection{Graph Inference}\label{sec:Graph Inference}

The final architectural component realizes the graph inference term $\log P(G \mid M)$, decoding the causal structure from the neurally encoded mechanism representations.

To produce a causal graph prediction that remains invariant to the sample size $N$, CDFM aggregates the sample-wise contextual representations into variable-level structural summaries. To learn causal structure from a diverse mechanism space $\mathcal{M}$, CDFM first employs an induced self-attention block with $2K$ learnable inducing vectors as $K$ experts, aiming to extract the causal mechanism Mixture-of-Experts (MoE) extraction process.

Subsequently, these intermediate expert representations are routed through an attention-based aggregator, condensing the diverse expert outputs down to two global summary vectors for each variable $i$: a cause representation $\mathbf{e}_i^{\mathrm{c}}$ and an effect representation $\mathbf{e}_i^{\mathrm{e}}$. This attention-driven compression effectively acts as a learned soft-routing mechanism, adaptively attending to the most relevant expert evidence based on the active underlying data-generating process.

These representations are evaluated to obtain the respective cause and effect vectors, yielding $\mathbf{u}_i = \mathbf{W}_c \mathbf{e}_i^{\mathrm{c}}$ and $\mathbf{v}_j = \mathbf{W}_e \mathbf{e}_j^{\mathrm{e}}$, where $\mathbf{u}_i, \mathbf{v}_j \in \mathbb{R}^p$. The continuous edge logit is then computed using a temperature-scaled dot product augmented by a global sparsity bias $A_{ij} = \exp(\tau) \cdot \frac{\mathbf{u}_i^{\mathsf{T}} \mathbf{v}_j}{\sqrt{p}} + b_0$, where $p$ is the projection dimension, $\tau$ is a learned log-temperature parameter governing the sharpness of the edge probabilities, and $b_0$ is a learned global bias representing the baseline sparsity prior of the causal graphs. The diagonal is strictly masked to $-\infty$ to explicitly preclude causal self-loops. The inherent asymmetry of this bilinear evaluation ($\mathbf{u}_i^{\mathsf{T}} \mathbf{v}_j \neq \mathbf{u}_j^{\mathsf{T}} \mathbf{v}_i$) naturally accommodates both directed causal relationships and bidirected edges.

Moreover, to convert the continuous edge probabilities into a binary graph, we employ a lightweight graph-adaptive threshold calibrator fitted on an independent synthetic calibration set. The calibrator uses ten graph-level features, including the mean and standard deviation, distributional quantiles, the logarithm of the number of variables, the median and maximum gaps between consecutively sorted probabilities, and a Kneedle-inspired knee statistic describing the separation between high-confidence edges and the low-probability tail~\citep{satopaa2011finding}. The graph-specific threshold is predicted by a logit-linear model whose coefficients are learned once from the oracle thresholds of held-out synthetic graphs and remain fixed during evaluation.

\subsection{Causal Prior Construction and Pretraining}
\label{sec:pretraining_prior}

The generalization capability of CDFM fundamentally depends on the diversity and identifiability of its pretraining data. We construct a massive-scale synthetic data generator that produces $(\mathbf{X}, G)$ pairs on the fly, where each training instance is sampled from an independently randomized causal model. The generator spans a broad spectrum of causal scenarios through the following dimensions:

\paragraph{Graph structure diversity.}
Causal graphs are sampled from six distinct random graph models to cover a wide range of topological patterns: Erd\H{o}s--R\'enyi (ER), scale-free (SF) and transposed scale-free (SFT), Watts--Strogatz small-world (WS), stochastic block model (SBM), and geometric random graph (GRG). The expected number of edges per variable is uniformly sampled from $\{1, 2, 3\}$, yielding graphs that range from sparse trees to moderately dense DAGs. The number of observed variables and sample size are drawn uniformly from $[2, 100]$ and $[50,4096]$, respectively.

\paragraph{Mechanism type diversity.}
The core of our data generator is a broad taxonomy of structural causal mechanisms, each encoding distinct functional-form and distribution. We categorize these mechanisms into several families:

\textit{Continuous additive mechanisms} include linear structural equation models (e.g., $X_j = \sum_{i \in \mathrm{PA}_j} w_{ij} X_i + N_j$) with both homoscedastic and heteroscedastic noise~\citep{yinEffectiveCausalDiscovery2024}, random Fourier feature (RFF) models that introduce highly nonlinear dependencies; post-nonlinear causal models~\citep{zhangIdentifiabilityPostNonlinear2009} where a nonlinear distortion is applied after additive mixing ($X_j = g_2(g_1(\mathbf{PA}_j) + N_j)$), causal additive models~\citep{buhlmannCAMCausalAdditive2014} where each parent contributes through an independent univariate nonlinearity, and physics-inspired mechanisms.

\textit{Discrete and mixed-type mechanisms} target the prevalent non-continuous data in real-world applications. These include conditional probability tables for purely categorical variables; discrete additive noise models~\citep{petersIdentifyingCauseEffectDiscrete2010} for ordinal discrete data; and ordinal causal discovery formulations~\citep{niOrdinalCausalDiscovery2022} that bridge continuous and discrete regimes within a single mechanism.

\textit{Post-processing transformations} simulate measurement and recording imperfections. Independent of the base SCM, each batch may undergo additive measurement error (perturbing a subset of variables with independent Gaussian noise) or discretization (rounding a subset of child variables to coarse bins). These transformations can co-occur within the same batch, producing data that reflects the messy, multi-step nature of real-world measurement.

\paragraph{Noise and edge-weight diversity.}
To ensure that CDFM can disentangle causal signal from exogenous variation, we vary both the noise distribution and the signal-to-noise ratio per mechanism instance. The exogenous noise $N_j$ for each variable is drawn from one of several different distribution families, such as Gaussian, Laplace, Cauchy, Beta, Exponential, Student's-$t$, log-normal, and Gaussian mixtures.

\paragraph{Latent confounding.}
Causal sufficiency rarely holds in practice. To equip CDFM with the ability to reason under unobserved confounding, $15\%$ of training batches include latent variables. The generator selects $d_{\mathrm{latent}}$ non-leaf nodes as latent, and then generates the full SCM, and finally drops the latent columns. In the resulting ground-truth graph, observed variable pairs that share a latent parent receive bidirected edges $X_i \leftrightarrow X_j$, while direct causal links retain directed edges $X_i \rightarrow X_j$. When a pair exhibits both a direct causal effect and a shared latent confounder, the directed edge takes priority. Acyclic constraint is disabled for these batches since bidirected edges inherently form cycles in the adjacency matrix.

\paragraph{Missing values.}
We further introduce random missingness to provide the auxiliary imputation task with a training signal and to prepare the model for incomplete real-world datasets. With $50\%$ probability, a batch undergoes masking where some portion of the entries are missing completely at random (MCAR). The original values are retained exclusively for computing the quantile imputation loss, while the model receives only the masked matrix with an explicit binary mask channel.

\subsection{Training Objective}

CDFM is trained end-to-end using a joint objective function that encompasses standard classification supervision, margin-based regularization, sparsity constraints, and self-supervised reconstruction terms:
\begin{equation}
    \mathcal{L}_{\mathrm{total}} = \mathcal{L}_{\mathrm{BCE}} + \alpha_{\mathrm{margin}} \cdot \mathcal{L}_{\mathrm{margin}} + \gamma_{\mathrm{F_1}} \cdot \mathcal{L}_{\mathrm{soft-F_1}} + \beta_{\mathrm{acyc}} \cdot \mathcal{L}_{\mathrm{acyc}} + \alpha_{\mathrm{impute}} \cdot \mathcal{L}_{\mathrm{quantile}},
\end{equation}
where $\alpha_{\mathrm{margin}}$, $\gamma_{\mathrm{F_1}}$, $\beta_{\mathrm{acyc}}$, and $\alpha_{\mathrm{impute}}$ balance the respective loss components.

\paragraph{Structural supervision.} 
To ensure reliable numerical convergence and sharp directional confidence, structural supervision combines standard binary cross-entropy ($\mathcal{L}_{\mathrm{BCE}}$) and a margin-based variant ($\mathcal{L}_{\mathrm{margin}}$). The BCE loss establishes baseline probabilistic alignment:
\begin{equation}
    \mathcal{L}_{\mathrm{BCE}} = -\frac{1}{|\mathcal{E}_{\mathrm{valid}}|} \sum_{(i,j) \in \mathcal{E}_{\mathrm{valid}}} \Big[ w_{\mathrm{pos}} \cdot G_{ij} \log \sigma(W_{ij}) + (1 - G_{ij}) \log(1 - \sigma(W_{ij})) \Big],
\end{equation}
where $w_{\mathrm{pos}}$ is a positive-class weight adjusted for graph sparsity. The margin-based component acts as a structural regularizer that explicitly penalizes ambiguous predictions near the classification threshold, utilizing a strict margin parameter $m$:
\begin{equation}
    \mathcal{L}_{\mathrm{margin}} = -\frac{1}{|\mathcal{E}_{\mathrm{valid}}|} \sum_{(i,j) \in \mathcal{E}_{\mathrm{valid}}} \Big[ w_{\mathrm{pos}} \cdot G_{ij} \log \sigma(W_{ij} - m) + (1 - G_{ij}) \log(1 - \sigma(W_{ij} + m)) \Big],
\end{equation}

\paragraph{Soft $F_1$ loss.} 
Because causal adjacency matrices are inherently sparse, standard cross-entropy objectives frequently converge to trivial sparse solutions. We incorporate a differentiable relaxation of the $F_1$-score:
\begin{equation}
    \mathcal{L}_{\mathrm{soft-F_1}} = 1 - \frac{2 \cdot \mathrm{TP}_{\mathrm{soft}}}{2 \cdot \mathrm{TP}_{\mathrm{soft}} + \mathrm{FP}_{\mathrm{soft}} + \mathrm{FN}_{\mathrm{soft}}},
\end{equation}
where the soft true positives ($\mathrm{TP}_{\mathrm{soft}}$), false positives ($\mathrm{FP}_{\mathrm{soft}}$), and false negatives ($\mathrm{FN}_{\mathrm{soft}}$) are evaluated continuously over the predicted sigmoid probabilities $P = \sigma(W)$ and the ground-truth matrix $G$:
\begin{equation}
    \mathrm{TP}_{\mathrm{soft}} = \sum_{i \neq j} P_{ij} G_{ij}, \quad \mathrm{FP}_{\mathrm{soft}} = \sum_{i \neq j} P_{ij} (1 - G_{ij}), \quad \mathrm{FN}_{\mathrm{soft}} = \sum_{i \neq j} (1 - P_{ij}) G_{ij}.
\end{equation}

\paragraph{Acyclicity constraint.} 
To enforce the structural validity of directed graphs, CDFM integrates a spectral radius acyclicity penalty $\mathcal{L}_{\mathrm{acyc}} = \max(0, \rho(\sigma(W)) - \delta)$, estimated dynamically via power iteration. This constraint is systematically omitted for batches containing latent confounding, where bidirected edges naturally permit cyclic representations in the observed mixed graph.

\paragraph{Quantile imputation loss.} 
Corresponding with the data reconstruction term $\log P(\mathbf{X} \mid G, M)$ defined within the variational lower bound, a pinball loss is evaluated at a set of $Q$ designated quantile levels (denoted by $q_k \in (0, 1)$ for $k=1, \dots, Q$) across the set of missing entries $\mathcal{R}$:
\begin{equation}
    \mathcal{L}_{\mathrm{quantile}} = \frac{1}{|\mathcal{R}| \cdot Q} \sum_{(i,j) \in \mathcal{R}} \sum_{k=1}^{Q} 2 \cdot (x_{ij}^{\mathrm{true}} - \hat{y}_{ij}^{q_k}) \cdot \big(q_k - \mathbf{1}[x_{ij}^{\mathrm{true}} \leq \hat{y}_{ij}^{q_k}]\big).
\end{equation}

\section{Experiments}\label{sec:experiments}

We organize the experiments around three central questions. First, can a single pretrained CDFM model recover causal structures across heterogeneous structural causal models, graph scales, and sample regimes without dataset-specific optimization? Second, can the causal representations learned from synthetic SCMs transfer to physical and empirical systems in the real-world? Third, does CDFM exhibit behavior consistent with established boundaries of causal identifiability under controlled conditions? We additionally evaluate the missing-value reconstruction head as an auxiliary capability arising from the structure-aware training objective.

\subsection{Experimental Setup}
\label{sec:experimental_setup}

For CDFM, all benchmark experiments are conducted in a zero-shot setting. The model is directly applied to each test dataset without access to its ground-truth graph and without dataset-specific parameter optimization, and a frozen graph-adaptive calibration is used for the decision boundary.

The controlled identifiability studies in Section~\ref{sec:exp_identifiability} follow a different protocol. For these diagnostic experiments, separately adapted copies of CDFM are trained on deliberately constructed low-dimensional SCM families. These experiments are intended to examine whether the model architecture can learn decision boundaries consistent with known identifiability results; they are not included in the zero-shot benchmark claims.

\paragraph{Baselines.}
We compare CDFM with representative causal discovery approaches covering different methodological paradigms. These include amortized causal discovery models (AVICI \cite{lorchAmortizedInferenceCausal2022}, TabCausal \cite{li2026tabcausal}, Arrow \cite{thompson2026arrow}), constraint-based methods (PC \cite{spirtesCausationPredictionSearch2000}), score-based methods (GES \cite{chickeringOptimalStructureIdentification2002}), functional causal models (DirectLiNGAM \cite{shimizuDirectLINGAMDirectMethod2011}), and continuous optimization methods (NOTEARS \cite{zhengDAGsNOTEARSContinuous2018}). All baselines are evaluated using their standard implementations with recommended hyperparameters.

\paragraph{Metrics.}
For multivariate causal discovery, we evaluate edge ranking using AUROC and thresholded graph recovery using precision, recall, $F_1$ score. Results are averaged over independently generated datasets.
For bivariate causal direction inference, we report weighted accuracy following the official evaluation protocol of the T\"ubingen benchmark.

\subsection{Synthetic Benchmark}
\label{sec:synthetic_benchmark}

The synthetic evaluation provides the primary controlled test of whether one frozen model can operate across a broad range of causal data-generating processes. 

\paragraph{Benchmark design.}
We construct a comprehensive synthetic evaluation suite spanning 15 distinct structural causal mechanism families (see Appendix~\ref{app:mechanism_names} for detailed descriptions). These families encompass linear and nonlinear structural equations, homoscedastic and heteroscedastic noise regimes, continuous and discrete variable types, post-nonlinear transformations, measurement error perturbations, ordinal observations, rounded discretization, and time-lagged causal mechanisms.

The benchmark systematically varies graph size across $D \in \{10, 15, 20, 25, 30, 40, 50, 70, 100\}$ variables and sample size across $N \in \{500, 1000, 2000, 3000, 4000\}$ observations. For each $(D, N, \text{mechanism})$ configuration, multiple independently sampled datasets are generated using held-out random seeds.

Tables~\ref{tab:synthetic_by_sample_auc} and~\ref{tab:synthetic_by_sample_f1} summarize the performance of all methods aggregated over all 15 mechanism families and all graph sizes, grouped by sample size. CDFM outperforms all baselines across every sample regime in both AUROC and $F_1$.

This behavior indicates that CDFM can effectively utilize additional observational information during inference rather than relying on fixed-size statistical heuristics. In contrast to classical methods that often require manually selected assumptions about the underlying mechanism class, CDFM maintains stable performance across different sample regimes.

\begin{table}[t]
\centering
\caption{Synthetic benchmark AUROC grouped by sample size. Results averaged over 15 mechanism families and ten graph sizes ($D \in \{10,15,20,25,30,40,50,70,100\}$).}
\label{tab:synthetic_by_sample_auc}
\begin{tabular}{l c c c c c}
\toprule
Method & 500 samples & 1000 samples & 2000 samples & 3000 samples & 4000 samples \\
\midrule
CDFM & \textbf{0.864 $\pm$ 0.109} & \textbf{0.878 $\pm$ 0.104} & \textbf{0.885 $\pm$ 0.101} & \textbf{0.887 $\pm$ 0.100} & \textbf{0.888 $\pm$ 0.101} \\
TabCausal & 0.767 $\pm$ 0.166 & 0.774 $\pm$ 0.166 & 0.782 $\pm$ 0.167 & 0.786 $\pm$ 0.164 & 0.786 $\pm$ 0.166 \\
Arrow & 0.676 $\pm$ 0.153 & 0.669 $\pm$ 0.149 & 0.680 $\pm$ 0.154 & 0.682 $\pm$ 0.154 & 0.685 $\pm$ 0.154 \\
AVICI & 0.756 $\pm$ 0.163 & 0.760 $\pm$ 0.165 & 0.764 $\pm$ 0.164 & 0.761 $\pm$ 0.161 & 0.760 $\pm$ 0.163 \\
DirectLiNGAM & 0.676 $\pm$ 0.150 & 0.689 $\pm$ 0.154 & 0.713 $\pm$ 0.158 & 0.724 $\pm$ 0.162 & 0.731 $\pm$ 0.164 \\
GES & 0.649 $\pm$ 0.105 & 0.657 $\pm$ 0.108 & 0.663 $\pm$ 0.110 & 0.667 $\pm$ 0.110 & 0.667 $\pm$ 0.112 \\
PC & 0.700 $\pm$ 0.122 & 0.715 $\pm$ 0.127 & 0.728 $\pm$ 0.129 & 0.737 $\pm$ 0.132 & 0.740 $\pm$ 0.133 \\
NOTEARS & 0.563 $\pm$ 0.127 & 0.581 $\pm$ 0.124 & 0.574 $\pm$ 0.151 & 0.583 $\pm$ 0.157 & 0.585 $\pm$ 0.167 \\
\bottomrule
\end{tabular}
\end{table}

\begin{table}[t]
\centering
\caption{Synthetic benchmark $F_1$ score grouped by sample size. Results averaged over 15 mechanism families and ten graph sizes ($D \in \{10,15,20,25,30,40,50,70,100\}$). Bold indicates best.}
\label{tab:synthetic_by_sample_f1}
\begin{tabular}{l c c c c c}
\toprule
Method & 500 samples & 1000 samples & 2000 samples & 3000 samples & 4000 samples \\
\midrule
CDFM & \textbf{0.552 $\pm$ 0.200} & \textbf{0.578 $\pm$ 0.197} & \textbf{0.587 $\pm$ 0.201} & \textbf{0.591 $\pm$ 0.202} & \textbf{0.591 $\pm$ 0.201} \\
TabCausal & 0.378 $\pm$ 0.246 & 0.396 $\pm$ 0.250 & 0.403 $\pm$ 0.258 & 0.405 $\pm$ 0.258 & 0.407 $\pm$ 0.260 \\
Arrow & 0.202 $\pm$ 0.155 & 0.196 $\pm$ 0.149 & 0.191 $\pm$ 0.155 & 0.188 $\pm$ 0.152 & 0.187 $\pm$ 0.153 \\
AVICI & 0.354 $\pm$ 0.260 & 0.360 $\pm$ 0.263 & 0.357 $\pm$ 0.266 & 0.338 $\pm$ 0.257 & 0.338 $\pm$ 0.260 \\
DirectLiNGAM & 0.381 $\pm$ 0.233 & 0.404 $\pm$ 0.244 & 0.405 $\pm$ 0.230 & 0.401 $\pm$ 0.216 & 0.392 $\pm$ 0.207 \\
GES & 0.370 $\pm$ 0.195 & 0.383 $\pm$ 0.196 & 0.389 $\pm$ 0.196 & 0.394 $\pm$ 0.195 & 0.395 $\pm$ 0.196 \\
PC & 0.471 $\pm$ 0.197 & 0.493 $\pm$ 0.196 & 0.513 $\pm$ 0.196 & 0.522 $\pm$ 0.195 & 0.526 $\pm$ 0.195 \\
NOTEARS & 0.138 $\pm$ 0.113 & 0.139 $\pm$ 0.110 & 0.141 $\pm$ 0.115 & 0.126 $\pm$ 0.110 & 0.126 $\pm$ 0.112 \\
\bottomrule
\end{tabular}
\end{table}

\begin{figure}[t]
    \centering
    \begin{subfigure}[t]{0.48\textwidth}
        \centering
        \includegraphics[width=\textwidth]{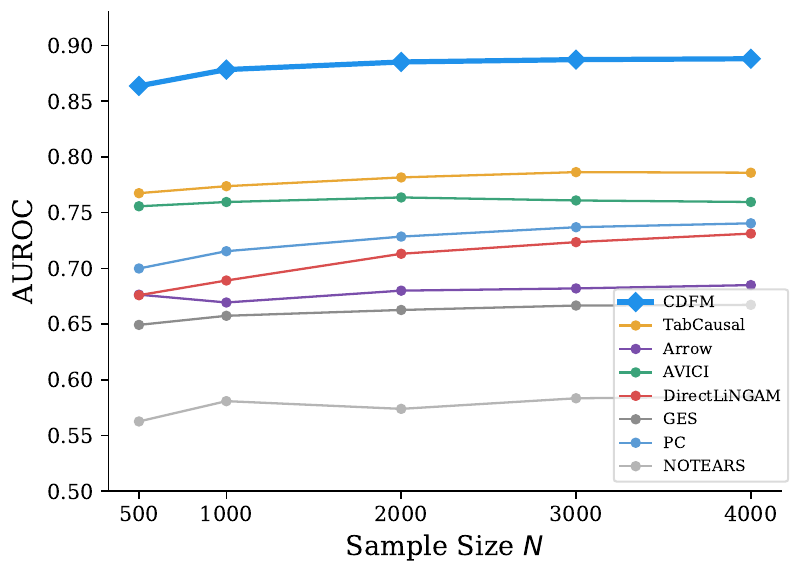}
        \caption{AUROC vs.\ sample size $N$.}
        \label{fig:scaling_a}
    \end{subfigure}
    \hfill
    \begin{subfigure}[t]{0.48\textwidth}
        \centering
        \includegraphics[width=\textwidth]{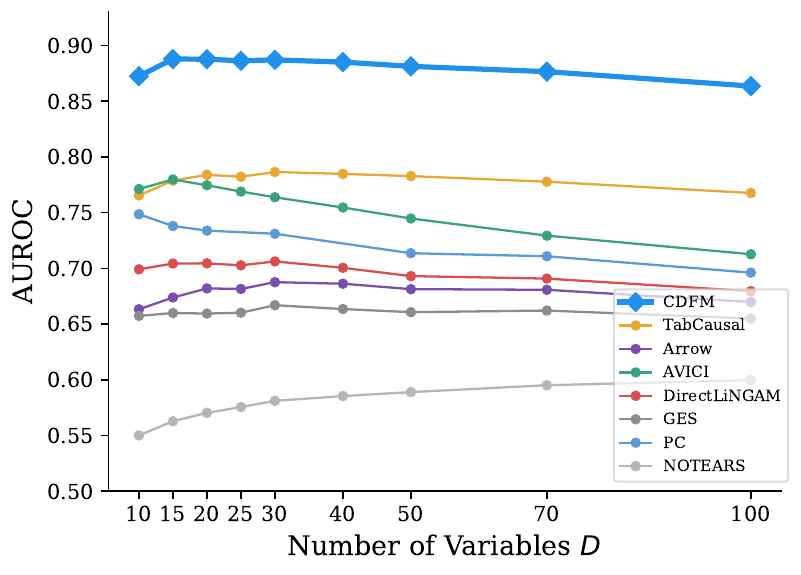}
        \caption{AUROC vs.\ number of variables $D$.}
        \label{fig:scaling_b}
    \end{subfigure}
    \caption{Performance over sample size and the number of variables. (a)~AUROC as a function of sample size $N$; (b)~AUROC as a function of the number of variables $D$. Results are averaged over all 15 mechanism families.}
    \label{fig:scaling}
\end{figure}

\begin{table}[t]
\centering
\caption{Synthetic robustness across 15 mechanism families. Results are averaged over nine graph sizes ($D \in \{10,15,20,25,30,40,50,70,100\}$) and five sample sizes ($N \in \{500,1000,2000,3000,4000\}$). Best baseline denotes the strongest non-CDFM method under each metric. Bold indicates best overall. Mechanism abbreviations are explained in Appendix~\ref{app:mechanism_names}.}
\label{tab:mechanism_robustness}
\resizebox{\textwidth}{!}{
\begin{tabular}{lcccccc}
\toprule
Mechanism
& CDFM AUROC$\uparrow$
& Best Baseline AUROC$\uparrow$
& AUROC baseline
& CDFM $F_1\uparrow$
& Best Baseline $F_1\uparrow$
& $F_1$ baseline \\
\midrule
CAM & \textbf{0.863 $\pm$ 0.085} & 0.794 $\pm$ 0.116 & TabCausal & \textbf{0.591 $\pm$ 0.156} & 0.508 $\pm$ 0.146 & PC \\
CPT & \textbf{0.767 $\pm$ 0.136} & 0.725 $\pm$ 0.122 & PC & 0.373 $\pm$ 0.212 & \textbf{0.468 $\pm$ 0.177} & PC \\
Discrete ANM & \textbf{0.881 $\pm$ 0.082} & 0.783 $\pm$ 0.094 & PC & \textbf{0.600 $\pm$ 0.171} & 0.598 $\pm$ 0.150 & PC \\
Linear & \textbf{0.939 $\pm$ 0.071} & 0.929 $\pm$ 0.089 & TabCausal & 0.668 $\pm$ 0.192 & \textbf{0.759 $\pm$ 0.190} & DirectLiNGAM \\
Linear-Hetero & \textbf{0.927 $\pm$ 0.074} & 0.904 $\pm$ 0.098 & TabCausal & 0.641 $\pm$ 0.187 & \textbf{0.729 $\pm$ 0.173} & DirectLiNGAM \\
Linear-Ordinal & \textbf{0.862 $\pm$ 0.113} & 0.714 $\pm$ 0.150 & PC & \textbf{0.553 $\pm$ 0.196} & 0.477 $\pm$ 0.208 & PC \\
Meas. Error & \textbf{0.882 $\pm$ 0.098} & 0.827 $\pm$ 0.128 & TabCausal & \textbf{0.591 $\pm$ 0.193} & 0.516 $\pm$ 0.191 & PC \\
Physical & \textbf{0.833 $\pm$ 0.097} & 0.763 $\pm$ 0.117 & TabCausal & \textbf{0.522 $\pm$ 0.187} & 0.446 $\pm$ 0.167 & PC \\
Physical-Hetero & \textbf{0.870 $\pm$ 0.081} & 0.740 $\pm$ 0.145 & TabCausal & \textbf{0.558 $\pm$ 0.173} & 0.442 $\pm$ 0.149 & PC \\
PNL & \textbf{0.901 $\pm$ 0.096} & 0.813 $\pm$ 0.133 & TabCausal & \textbf{0.589 $\pm$ 0.211} & 0.453 $\pm$ 0.218 & PC \\
RFF & \textbf{0.947 $\pm$ 0.054} & 0.918 $\pm$ 0.090 & TabCausal & \textbf{0.680 $\pm$ 0.178} & 0.647 $\pm$ 0.204 & TabCausal \\
RFF-Hetero & \textbf{0.938 $\pm$ 0.057} & 0.906 $\pm$ 0.096 & AVICI & \textbf{0.654 $\pm$ 0.178} & 0.647 $\pm$ 0.215 & AVICI \\
RFF-Ordinal & \textbf{0.860 $\pm$ 0.100} & 0.736 $\pm$ 0.135 & PC & \textbf{0.547 $\pm$ 0.188} & 0.512 $\pm$ 0.176 & PC \\
Rounded & \textbf{0.881 $\pm$ 0.096} & 0.787 $\pm$ 0.136 & TabCausal & \textbf{0.575 $\pm$ 0.193} & 0.456 $\pm$ 0.196 & PC \\
Time-Lag & \textbf{0.854 $\pm$ 0.115} & 0.757 $\pm$ 0.147 & TabCausal & \textbf{0.558 $\pm$ 0.191} & 0.463 $\pm$ 0.189 & PC \\
\bottomrule
\end{tabular}
}
\end{table}

\paragraph{Performance across sample sizes and graph sizes.} 
Figure~\ref{fig:scaling_a} visualizes these scaling trends across sample sizes from $N = 500$ to $N = 4000$, and CDFM consistently outperforms all baselines in both regimes, with its advantage widening as the problem scale increases.

Figure~\ref{fig:scaling_b} evaluates performance as the number of variables increases from $D = 10$ to $D=100$. All methods become less accurate as the structural search space grows, but CDFM degrades more gradually than the baselines. At $D=100$, CDFM retains an AUROC above $0.87$, whereas the strongest competing amortized model reaches approximately $0.75$.

The amortized methods (AVICI, TabCausal, Arrow) show similar scaling trends to CDFM but at consistently lower absolute performance levels. Classical constraint-based and score-based methods (PC, GES) exhibit steeper performance degradation due to the combinatorial explosion of conditional independence tests and score evaluations in high dimensions.

\paragraph{Performance across causal mechanisms.}
\begin{figure}[t]
    \centering
    \includegraphics[width = \textwidth]{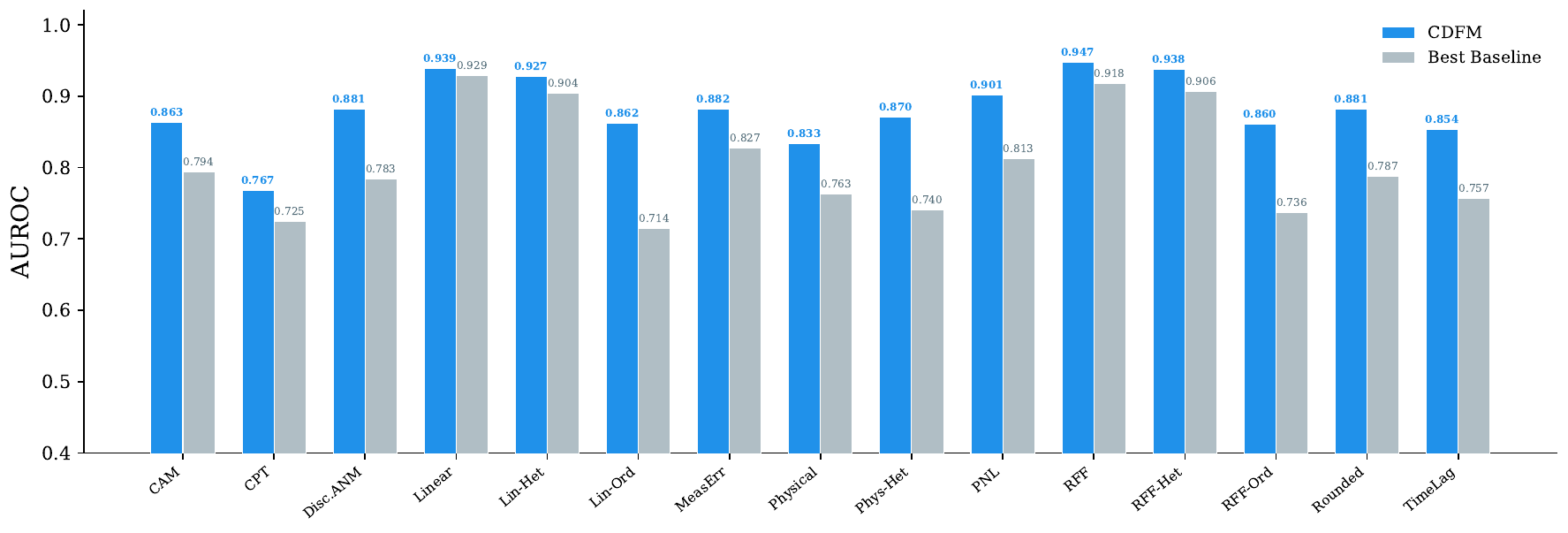}
    \caption{\textbf{Mechanism robustness.} AUROC comparison between CDFM and the best-performing baseline for each of the 15 mechanism families. CDFM wins on all 15 mechanisms in AUROC.}
    \label{fig:mechanism_robustness}
\end{figure}

\begin{figure}[t]
    \centering
    \includegraphics[width = \textwidth]{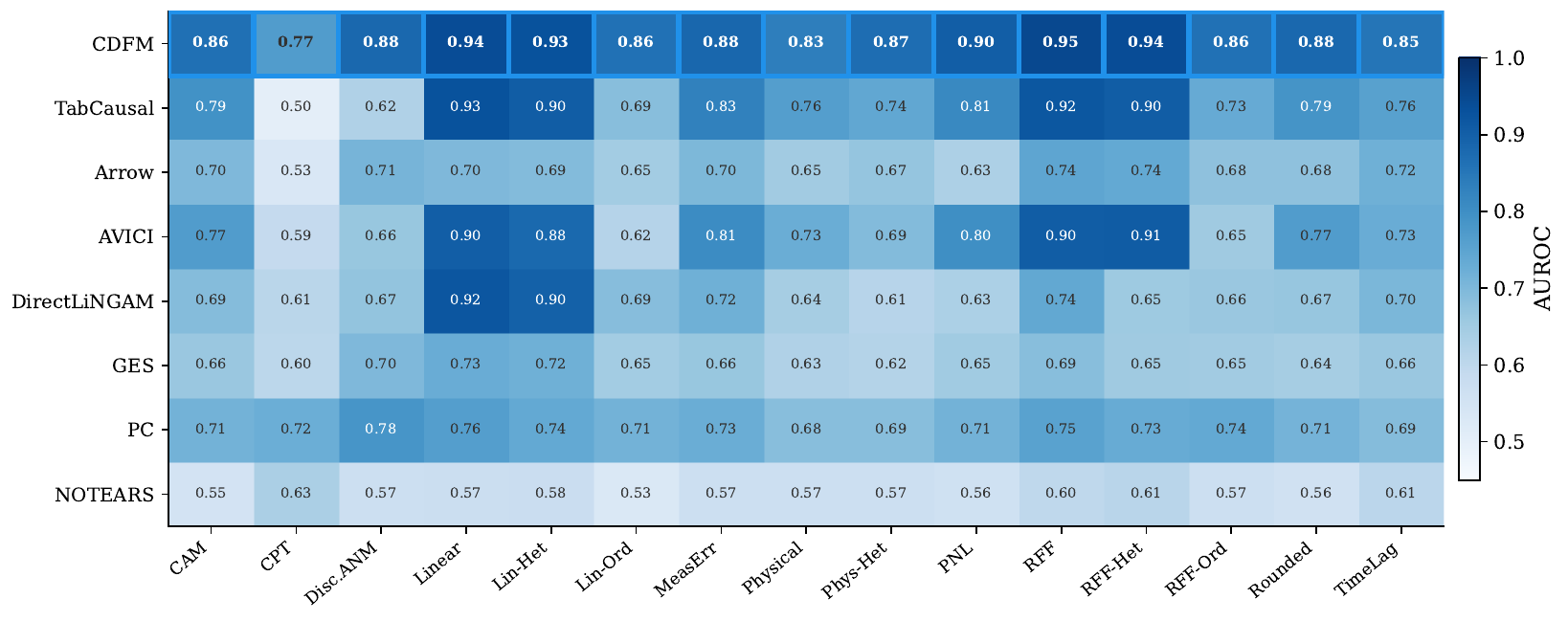}
    \caption{\textbf{Method $\times$ mechanism AUROC heatmap.} Rows correspond to methods, columns to mechanism families. CDFM (top row) achieves the highest AUROC across all 15 mechanisms.}
    \label{fig:mechanism_heatmap}
\end{figure}

To further examine whether the observed performance originates from broad mechanism coverage rather than a small subset of favorable cases, we evaluate performance separately for each mechanism family. Figure~\ref{fig:mechanism_robustness} provides a visual comparison of CDFM against the best baseline per mechanism. The mechanism heatmap in Figure~\ref{fig:mechanism_heatmap} further reveals that CDFM maintains a broad advantage across the evaluated causal mechanisms.

Specifically, as shown in Table~\ref{tab:mechanism_robustness}, CDFM achieves competitive or superior performance across most nonlinear, discrete, and corrupted observation settings. Specialized methods such as DirectLiNGAM remain advantageous in specific linear non-Gaussian scenarios, demonstrating the benefit of strong inductive assumptions when their assumptions are exactly satisfied. However, CDFM provides a more balanced performance profile across heterogeneous environments, suggesting that its advantage comes from mechanism-level generalization rather than optimization toward a single causal model class.

\subsection{Real-world Benchmark}
\label{sec:real_benchmark}

\begin{table}[t]
\centering
\caption{Causal Chamber Task A1 results for multivariate DAG recovery.}
\label{tab:causal_chamber_a1}
\begin{tabular}{lccccc}
\toprule
Method & AUROC$\uparrow$ & $F_1\uparrow$ & Precision$\uparrow$ & Recall$\uparrow$ & SHD$\downarrow$ \\
\midrule
CDFM  & \textbf{0.952} & \textbf{0.727} & \textbf{0.889} & 0.615 & \textbf{17}  \\
TabCausal    & 0.930 & 0.603 & 0.792 & 0.487 & 25 \\
DirectLiNGAM & 0.775 & 0.495 & 0.403 & \textbf{0.641} & 44 \\
GES          & 0.699 & 0.426 & 0.591 & 0.333 & 35  \\
PC   & 0.568 & 0.281 & 0.360 & 0.231 & 43  \\
Arrow      & 0.848 & 0.255 & 0.438 & 0.180 & 40 \\
NOTEARS      & 0.380 & 0.095 & 0.057 & 0.282 & 159 \\
AVICI        & 0.423 & 0.000 & 0.000 & 0.000 & 41 \\
\bottomrule
\end{tabular}
\end{table}

\begin{table}[t]
\centering
\caption{Causal Chamber Task A2 results. The predictions from 20 separately processed environments are averaged.}
\label{tab:causal_chamber_a2}
\begin{tabular}{lccccc}
\toprule
Method & AUROC$\uparrow$ & $F_1\uparrow$ & Precision$\uparrow$ & Recall$\uparrow$ & SHD$\downarrow$  \\
\midrule
CDFM  & \textbf{0.965} & \textbf{0.727} & 0.889 & 0.615 & \textbf{18}  \\
TabCausal    & 0.939 & 0.635 & 0.833 & 0.513 & 23  \\
GES          & 0.906 & 0.633 & \textbf{0.905} & 0.487 & 22  \\
DirectLiNGAM & 0.813 & 0.349 & 0.226 & \textbf{0.769} & 82  \\
PC   & 0.616 & 0.250 & 0.412 & 0.180 & 40  \\
Arrow      & 0.797 & 0.170 & 0.500 & 0.103 & 39  \\
NOTEARS      & 0.446 & 0.127 & 0.073 & 0.513 & 188  \\
AVICI        & 0.503 & 0.000 & 0.000 & 0.000 & 40  \\
\bottomrule
\end{tabular}
\end{table}

\begin{table}[t]
\centering
\caption{T\"ubingen results on 95 bivariate cause-effect pairs.}
\label{tab:tuebingen}
\begin{tabular}{lc}
\toprule
Method  & Accuracy$\uparrow$  \\
\midrule
CDFM  & \textbf{0.671}  \\
TabCausal  & 0.638  \\
Arrow  & 0.602  \\
AVICI  & 0.488  \\
DirectLiNGAM  & 0.508  \\
NOTEARS  & 0.401  \\

\bottomrule
\end{tabular}
\end{table}

We next examine whether the learned causal representation transfers beyond this synthetic evaluation environment. We consider a physical multivariate system from Causal Chamber and heterogeneous bivariate datasets from the T\"ubingen cause--effect benchmark.

\paragraph{Causal Chamber.}
Causal Chamber contains experimentally validated causal discovery tasks generated from automatically controlled physical systems. We evaluate CDFM on the light-tunnel data using Tasks A1 and A2.

For Task A1, we use the observational environment and infer the directed causal graph from measured variables. For Task A2, the dataset contains multiple intervention environments. Since CDFM receives a single observational matrix as input and does not explicitly model intervention targets, we introduce an environment-ensemble adaptation where each environment is processed independently, and predicted edge probabilities are averaged before evaluation. Table~\ref{tab:causal_chamber_a1} and Table~\ref{tab:causal_chamber_a2} summarize the results of Task A1 and Task A2, respectively.

CDFM achieves strong performance on both tasks, demonstrating that a model pretrained entirely on synthetic causal mechanisms can transfer to real physical systems and demonstrate robustness under diverse causal mechanisms.

\paragraph{T\"ubingen cause-effect pairs.}
We further evaluate CDFM on the T\"ubingen cause-effect pairs benchmark, which contains heterogeneous real-world bivariate datasets collected from multiple scientific domains. We use 95 causal pairs for strictly bivariate tasks.

For each pair, CDFM receives only two observed variables and predicts the causal direction.
Table~\ref{tab:tuebingen} shows that CDFM achieves competitive performance with existing state-of-the-art methods. The results demonstrate that the learned causal representation transfers beyond synthetic environments and remains effective across diverse real-world domains.

\subsection{Empirical Validation of Identifiability}\label{sec:exp_identifiability}

We next conduct controlled diagnostic experiments to examine whether CDFM can learn to examine whether CDFM behaves consistently with theoretical identifiability boundaries. If the training data spans identifiable causal mechanisms, the model should learn to distinguish causal directions where identifiability holds, and should produce near-uniform predictions where it does not. In this section, we empirically validate these claims through a series of controlled experiments. Moreover, to ensure CDFM has been exposed to the full range of diagnostic tasks, we will fine-tune CDFM on these specific tasks.

\paragraph{Functional Asymmetry.} We first explore functional asymmetry in identifiability. We consider a family of bivariate additive-noise formulation $y = x + b \cdot x^3 + n$, where $b \in [-1, 1]$ controls the degree of nonlinearity and $q \in [0.5, 2.0]$ controls the non-Gaussianity of the noise distribution via the transformation $v \mapsto \text{sign}(v) \cdot |v|^q$ applied independently to both the cause $x \sim \mathcal{N}(0,1)$ and the noise $n \sim \mathcal{N}(0,1)$. At $b=0, q=1$, the model reduces to a linear Gaussian SCM, which is theoretically non-identifiable in both directions. Our goal is to see whether the foundation also adopts the theoretical results.

To ensure CDFM has been exposed to the full range of $(b, q)$ configurations, we fine-tune the pretrained model on bivariate ANM data, sampling $b \sim \mathcal{U}[-1, 1]$ and $q \sim \mathcal{U}[0.5, 2.0]$ independently per batch, with sample sizes $N \sim \mathcal{U}[100, 500]$.

\paragraph{Results.} Figure~\ref{fig:identifiability} presents the empirical findings on non-Gaussianity and nonlinearity. Figure~\ref{fig:identifiability_a} shows accuracy as a function of the non-Gaussianity parameter $q$ at $b=0$ (linear SCM). At $q = 1.0$ (Gaussian noise), the model achieves an accuracy around $0.5$, corresponding to chance-level performance, confirming that CDFM still obeys the theoretical results in the linear Gaussian case. As $q$ moves away from $1$, accuracy rises sharply to $1.00$ for both sub-Gaussian and super-Gaussian regimes. Figure~\ref{fig:identifiability_b} demonstrates the effect of nonlinearity at $q = 1$ (Gaussian noise). At $b = 0.0$ (purely linear), accuracy again remains at around $0.5$, while a slight nonlinearity ($|b| \geq 0.1$) suffices for identification ($0.97$--$1.00$). These results align with the theoretical results of the nonlinear additive noise model \cite{hoyerNonlinearCausalDiscovery2009}.

\begin{figure}[t]
    \centering
    \begin{subfigure}[t]{0.48\textwidth}
        \centering
        \includegraphics[width=\textwidth]{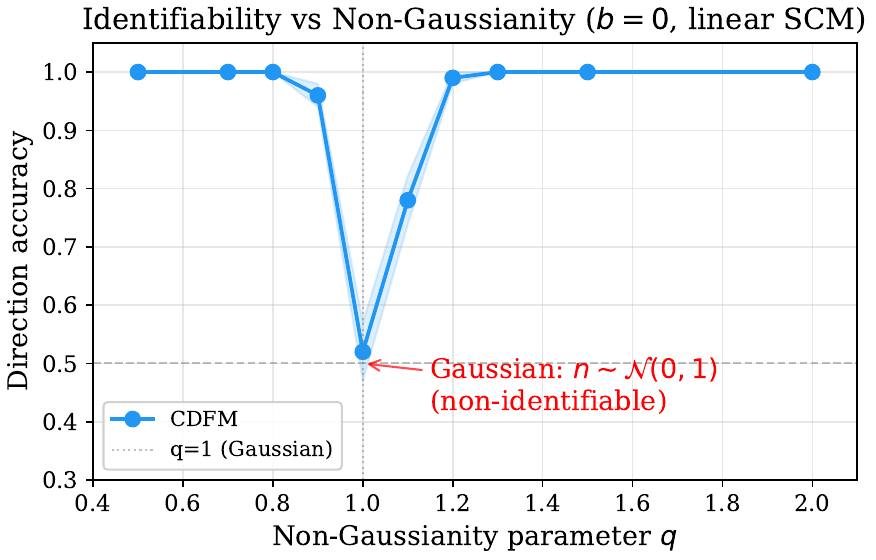}
        \caption{Non-Gaussianity parameter $q$ for linear SCMs ($b=0$).}
        \label{fig:identifiability_a}
    \end{subfigure}
    \hfill
    \begin{subfigure}[t]{0.48\textwidth}
        \centering
        \includegraphics[width=\textwidth]{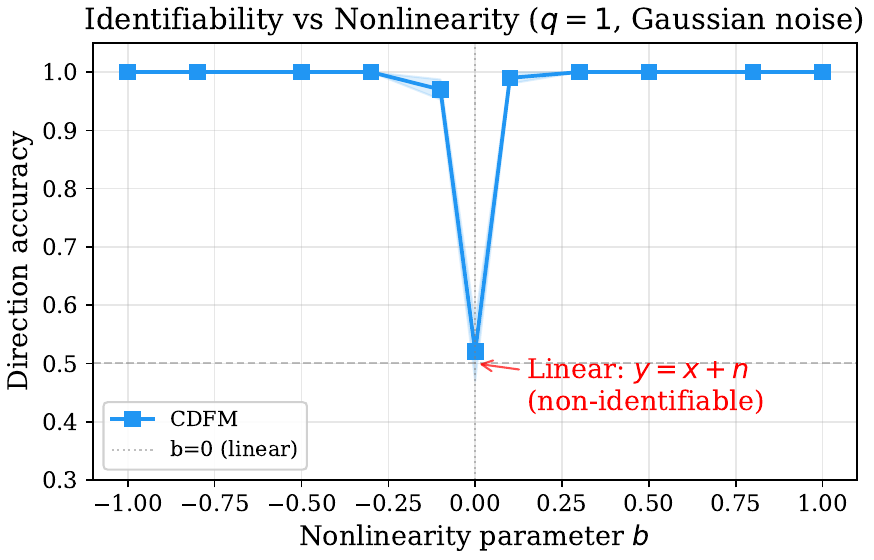}
        \caption{Nonlinearity parameter $b$ with Gaussian noise ($q=1$).}
        \label{fig:identifiability_b}
    \end{subfigure}
    \\[6pt]
    \begin{subfigure}[t]{0.48\textwidth}
        \centering
        \includegraphics[width=\textwidth]{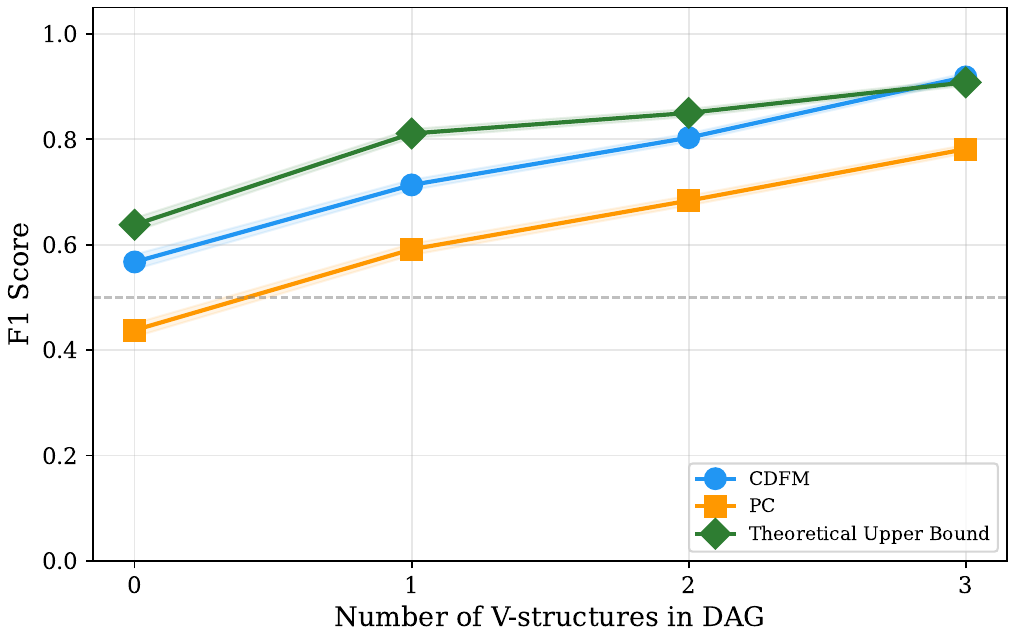}
        \caption{$F_1$ score vs.\ number of V-structures ($d=5$).}
        \label{fig:identifiability_c}
    \end{subfigure}
    \hfill
    \begin{subfigure}[t]{0.48\textwidth}
        \centering
        \includegraphics[width=\textwidth]{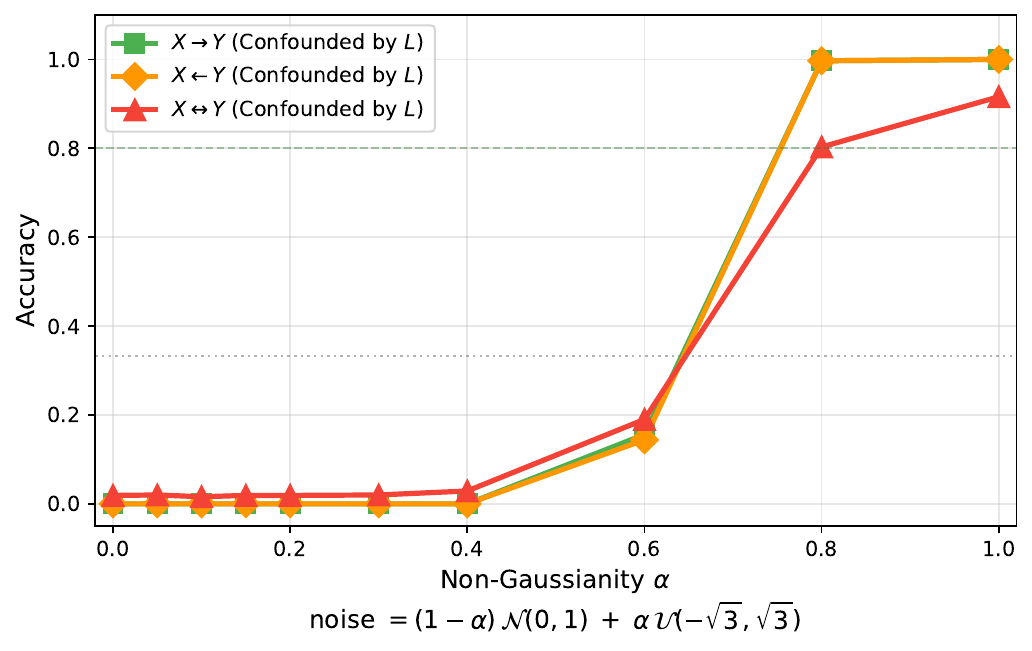}
        \caption{Latent confounding identifiability.}
        \label{fig:identifiability_d}
    \end{subfigure}
    \caption{\textbf{Empirical validation of identifiability.} (a)~Effect of non-Gaussianity on causal direction identification for linear SCMs. (b)~Effect of nonlinearity under Gaussian noise. (c)~$F_1$ score as a function of the number of V-structures, compared against PC and the theoretical upper bound. (d)~Three-class accuracy under latent confounding as a function of non-Gaussianity.}
    \label{fig:identifiability}
\end{figure}

\paragraph{Identifiability Through V-Structure.}

Under the causal Markov and faithfulness assumptions, observational conditional-independence relations generally identify a Markov equivalence class rather than a unique DAG, while unshielded colliders (also known as V-structures) provide compelled edge orientations, which offer an additional source of asymmetry beyond the functional asymmetry. Therefore, we further investigate whether CDFM can leverage causal information implied by conditional independence constraints.

To do so, we exclude additional directional information from functional asymmetry by conducting this experiment exclusively on linear Gaussian structural equation models. Therefore, any identifiable orientation must arise from conditional-independence constraints.

Then, we exhaustively enumerate all non-isomorphic DAGs on $d=5$ under sample size = 100 and group them according to their actual number of V-structures. We evaluate the $F_1$ score over CDFM, PC, and the Theoretical Upper Bound, which uses perfect d-separation on the true DAG, and the undetermined edges will be randomly oriented, providing a fairer comparison. To ensure the theoretical boundary is practically detectable, CDFM is fine-tuned on synthetic datasets generated from the same linear Gaussian model family.

\paragraph{Results.} Figure~\ref{fig:identifiability_c} presents the results of $F_1$ against the number of V-structures. CDFM's orientation performance improves as the number of collider structures increases, while PC exhibits the same qualitative tendency but is additionally affected by finite-sample errors in its conditional-independence tests. Moreover, CDFM approaches the theoretical upper bound, empirically confirming the theoretical identifiability boundary, showing that CDFM remains subject to classical causal identifiability theory.

\paragraph{Latent Confounding Identifiability.}
Finally, we examine whether CDFM can distinguish among three distinct causal scenarios under latent confounding: (i)~no direct edge between $X$ and $Y$, with only a latent confounder $L$ causing both ($X \leftarrow L \rightarrow Y$); (ii)~$X \rightarrow Y$ with confounding; and (iii)~$Y \rightarrow X$ with confounding. The model observes only the bivariate pair $(X, Y)$ and must classify the underlying relationship. For synthetic data, we follow the functional asymmetry with linear causal relationship $Y = aX+n$ where we control the non-Gaussianity of the noise distribution via varying $q$ by the transformation $v \mapsto \text{sign}(v) \cdot |v|^q$.

Figure~\ref{fig:identifiability_d} presents the 3-class classification accuracy for each scenario as a function of $q$. Two notable patterns emerge. First, the \textit{no-edge} scenario (pure confounding) is best identified in the sub-Gaussian regime ($q = 0.5$, accuracy $0.86$), where the distributional signature of a shared latent cause is most distinct from a direct causal edge. Accuracy degrades as $q$ approaches $1.0$ (Gaussian, accuracy $0.64$) and recovers partially for super-Gaussian noise. Second, the \textit{directed-edge} scenarios ($X \rightarrow Y$ and $Y \rightarrow X$) achieve their best accuracy at extreme super-Gaussian values ($q = 2.0$, accuracy $0.66$), suggesting that heavy-tailed non-Gaussianity provides the strongest signal for distinguishing the causal direction under confounding. Overall, these results demonstrate that CDFM exploits non-Gaussian distributional signatures to distinguish the three simulated relation types, which is consistent with classic identifiability theory.

These experiments collectively suggest that the theoretical boundary of CDFM is consistent with the classical identifiability theory.

\subsection{Auxiliary Analysis of Missing Value Imputation}

CDFM includes a quantile-reconstruction head to provide a self-supervised signal for learning conditional structure. Therefore, as a by-product of this work, we evaluate the imputation capability of CDFM on both continuous and discrete data. We consider two synthetic mechanisms: \textbf{continuous} data generated by the RFF mechanism, which uses nonlinear additive noise with random Fourier features, and \textbf{discrete} data generated by the CPT mechanism, which uses conditional probability tables with values in $\{0,1,2,3\}$.
For each mechanism, we use synthetic datasets, spanning $D\in [5,50]$ variables,
$N\in \{500,1000,2000\}$ samples, and 10 random seeds for each configuration. Missing values are generated under an MCAR scheme with missing rates in $\{0.1,0.2,0.3\}$.

CDFM performs imputation in a single forward pass using its built-in imputation head. Specifically, we extract the median quantile from its 21-level quantile prediction.
We compare CDFM against eight standard imputation methods:
(1)~\textbf{Mean} and (2)~\textbf{Median} column-wise imputation;
(3)~\textbf{KNN} ($k{=}5$) \cite{troyanskaya2001missing,pedregosa2011scikit};
(4)~\textbf{MICE} with BayesianRidge regression for 10 iterations \cite{vanbuuren2011mice};
(5)~\textbf{MissForest} \cite{stekhoven2012missforest} using RandomForest with 50 trees and 5 iterations;
(6)~\textbf{SoftImpute}, (7)~\textbf{IterativeSVD} \cite{rubinsteyn_fancyimpute}, and
(8)~\textbf{MatrixFactorization} \cite{rubinsteyn_fancyimpute}.
All methods are evaluated using mean absolute error (MAE) between the imputed and ground-truth values at the masked positions.

Table~\ref{tab:imputation} reports the MAE of all methods on both continuous (RFF)
and discrete (CPT) data.
CDFM one-shot imputation achieves the best overall performance in both regimes,
with an MAE of 0.452 on continuous data and 0.717 on discrete data.
Matrix-factorization-based methods (SoftImpute, IterativeSVD, MatrixFactorization)
perform poorly on discrete data, as their low-rank assumptions are violated by the
categorical structure.
Across missing rates, CDFM remains the top method for both regimes, with its advantage
over competitors widening as the missing rate increases from 0.1 to 0.3.

\begin{table}[t]
\centering
\caption{\textbf{Missing value imputation MAE $\downarrow$.}
CDFM vs. eight classical imputation methods on continuous (RFF) and
discrete (CPT) synthetic data.}
\label{tab:imputation}
\begin{tabular}{lcc}
\toprule
 & \textbf{Continuous} & \textbf{Discrete} \\
\midrule
CDFM                & \textbf{0.452$\pm$0.164} & \textbf{0.717$\pm$0.102} \\
MissForest          & 0.475$\pm$0.188 & 0.748$\pm$0.095 \\
MICE    & 0.500$\pm$0.163 & 0.760$\pm$0.097 \\
SoftImpute          & 0.505$\pm$0.150 & 0.834$\pm$0.117 \\
KNN ($k{=}5$)      & 0.573$\pm$0.184 & 0.780$\pm$0.095 \\
MatrixFactorization & 0.694$\pm$0.198 & 0.990$\pm$0.179 \\
Median              & 0.720$\pm$0.069 & 0.721$\pm$0.100 \\
Mean                & 0.727$\pm$0.067 & 0.779$\pm$0.093 \\
IterativeSVD        & 0.732$\pm$0.302 & 1.217$\pm$0.315 \\
\bottomrule
\end{tabular}
\end{table}

\begin{table}[t]
\centering
\caption{\textbf{Imputation MAE by missing rate.}
CDFM maintains its advantage across all missing rates on both continuous and discrete data.}
\label{tab:imputation_rate}
\resizebox{\textwidth}{!}{
\begin{tabular}{lccc|ccc}
\toprule
 & \multicolumn{3}{c}{\textbf{Continuous}} & \multicolumn{3}{c}{\textbf{Discrete}} \\
\cmidrule(lr){2-4} \cmidrule(lr){5-7}
\textbf{Method} & \textbf{0.1} & \textbf{0.2} & \textbf{0.3} & \textbf{0.1} & \textbf{0.2} & \textbf{0.3} \\
\midrule
CDFM                & \textbf{0.436$\pm$0.164} & \textbf{0.450$\pm$0.165} & \textbf{0.469$\pm$0.162} & \textbf{0.712$\pm$0.104} & \textbf{0.717$\pm$0.101} & \textbf{0.721$\pm$0.100} \\
MissForest          & 0.456$\pm$0.185 & 0.474$\pm$0.188 & 0.496$\pm$0.189 & 0.735$\pm$0.095 & 0.747$\pm$0.093 & 0.763$\pm$0.096 \\
IterativeImputer    & 0.463$\pm$0.162 & 0.502$\pm$0.161 & 0.536$\pm$0.158 & 0.752$\pm$0.097 & 0.754$\pm$0.095 & 0.775$\pm$0.098 \\
SoftImpute          & 0.487$\pm$0.153 & 0.505$\pm$0.151 & 0.524$\pm$0.146 & 0.823$\pm$0.113 & 0.835$\pm$0.117 & 0.845$\pm$0.120 \\
KNN ($k{=}5$)      & 0.541$\pm$0.192 & 0.572$\pm$0.184 & 0.605$\pm$0.172 & 0.770$\pm$0.094 & 0.782$\pm$0.096 & 0.788$\pm$0.095 \\
MatrixFactorization & 0.674$\pm$0.202 & 0.689$\pm$0.197 & 0.716$\pm$0.193 & 0.984$\pm$0.185 & 0.990$\pm$0.181 & 0.997$\pm$0.172 \\
Median              & 0.718$\pm$0.069 & 0.721$\pm$0.069 & 0.720$\pm$0.068 & 0.722$\pm$0.101 & 0.721$\pm$0.099 & 0.720$\pm$0.099 \\
Mean                & 0.726$\pm$0.068 & 0.728$\pm$0.068 & 0.728$\pm$0.066 & 0.779$\pm$0.093 & 0.780$\pm$0.094 & 0.778$\pm$0.092 \\
IterativeSVD        & 0.703$\pm$0.302 & 0.730$\pm$0.306 & 0.762$\pm$0.298 & 1.232$\pm$0.350 & 1.212$\pm$0.309 & 1.208$\pm$0.283 \\
\bottomrule
\end{tabular}}
\end{table}

\section{Conclusions}
In this work, we introduce CDFM, a causal foundation model that reframes causal discovery as inference over uncertain mechanisms rather than optimization under a single assumed causal model. By integrating mechanism inference, structural prediction, and reconstruction-based representation learning, CDFM learns transferable causal representations from diverse synthetic causal environments.

Extensive experiments demonstrate that CDFM achieves strong generalization across heterogeneous synthetic mechanisms, graph structures, and sample sizes, while maintaining competitive performance on real-world causal benchmarks, including physical causal systems and heterogeneous cause-effect datasets. Controlled identifiability experiments further show that CDFM follows theoretically meaningful causal behavior: it exploits available functional and graphical asymmetries while remaining uncertain in observationally non-identifiable regimes.

These findings suggest a promising direction for causal discovery: rather than designing increasingly specialized algorithms for individual assumptions, future causal models may benefit from learning broad representations over families of causal mechanisms. CDFM represents an initial step toward such general-purpose causal reasoning systems.

\bibliographystyle{abbrv}
\bibliography{ref}

\newpage

\appendix

\section{Proofs and Derivation}

\subsection{Derivation of the Variational ELBO}
\label{app:elbo_derivation}
To approximate the intractable marginal likelihood where the mechanism $M$ is integrated out, we introduce a variational distribution $Q(M|\mathbf{X})$ over the mechanism space $\mathcal{M}$. The joint log-probability $\log P(G, \mathbf{X})$ can be expanded and bounded as follows:
\begin{align}
\log P(G, \mathbf{X}) &= \log \int_{\mathcal{M}} P(G, M, \mathbf{X}) dM \\
&= \log \int_{\mathcal{M}} P(G, \mathbf{X} | M) P(M) dM \\
&= \log \int_{\mathcal{M}} P(G, \mathbf{X} | M) \frac{P(M)}{Q(M | \mathbf{X})} Q(M | \mathbf{X}) dM \\
&\ge \mathbb{E}_{Q(M | \mathbf{X})} \left[ \log \left( P(G, \mathbf{X} | M) \frac{P(M)}{Q(M | \mathbf{X})} \right) \right] \quad \text{(by Jensen's Inequality)} \\
&= \mathbb{E}_{Q(M | \mathbf{X})} \big[ \log P(G, \mathbf{X} | M) \big] - \mathbb{E}_{Q(M | \mathbf{X})} \left[ \log \frac{Q(M | \mathbf{X})}{P(M)} \right] \\
&= \mathbb{E}_{Q(M | \mathbf{X})} \big[ \log P(\mathbf{X} | G, M) + \log P(G | M) \big] - D_{\mathrm{KL}}\big( Q(M | \mathbf{X}) \parallel P(M) \big).
\end{align}
This completes the derivation of Eq.~\eqref{eq:elbo} in the main text.

\subsection{Proof of Theorem \ref{thm:impossibility}}

\begin{proof}
To provide an intuitive construction, we first assume that $X$ and $Y$ are continuous random variables. We will demonstrate that a valid Structural Causal Model can always be constructed for the causal direction $X \to Y$. 

Let $F_{Y|X}(y|x)$ denote the conditional cumulative distribution function of $Y$ given $X$. According to the inverse transform sampling theorem, there exists an auxiliary noise variable $N_Y \sim \mathcal{U}[0,1]$ that is independent of $X$, such that $Y = F^{-1}_{Y|X}(N_Y \mid X)$. Therefore, by defining the structural function as $f_Y(X, N_Y) := F^{-1}_{Y|X}(N_Y \mid X)$, we obtain $Y = f_Y(X, N_Y)$ with $N_Y \perp\!\!\!\perp X$. Hence, a valid SCM corresponding to the causal direction $X \to Y$ exists.

Applying the same argument to the reverse direction, there exists an independent noise variable $N_X \sim \mathcal{U}[0,1]$ such that $X = F^{-1}_{X|Y}(N_X \mid Y)$. Equivalently, defining $f_X(Y, N_X) := F^{-1}_{X|Y}(N_X \mid Y)$, we obtain $X = f_X(Y, N_X)$ where $N_X \perp\!\!\!\perp Y$. Therefore, for any continuous joint distribution $P_{X,Y}$, valid SCMs can be constructed for both causal directions $X \to Y$ and $Y \to X$. 

For more general cases, including discrete or mixed variables, the Functional Representation Lemma \cite{ElGamal2011Network} guarantees the existence of such independent uniform noise variables and measurable structural functions, yielding the exact same result. Consequently, without additional assumptions on the causal mechanisms, the causal direction is inherently unidentifiable from observational data alone.

\end{proof}

\subsection{Proof of Theorem \ref{thm:kl_known_mechanism}}

\begin{proof}
By definition, the expected log-likelihood of a graph $G$ under the true data-generating distribution is
\[
\mathcal{L}(G)
=
\mathbb{E}_{\mathbf{X}\sim P(\mathbf{X}|G^*,M^*)}
\left[
\log P(\mathbf{X}|G,M^*)
\right].
\]
Therefore, the difference between the expected log-likelihood of the true graph and that of an arbitrary alternative graph $G'$ is
\[
\begin{aligned}
\mathcal{L}(G^*)-\mathcal{L}(G')
&=
\mathbb{E}_{\mathbf{X}\sim P(\mathbf{X}|G^*,M^*)}
\left[
\log P(\mathbf{X}|G^*,M^*)
-
\log P(\mathbf{X}|G',M^*)
\right]\\
&=
\mathbb{E}_{\mathbf{X}\sim P(\mathbf{X}|G^*,M^*)}
\left[
\log
\frac{
P(\mathbf{X}|G^*,M^*)
}{
P(\mathbf{X}|G',M^*)
}
\right]\\
&=
D_{\mathrm{KL}}
\left[
P(\mathbf{X}|G^*,M^*)
\parallel
P(\mathbf{X}|G',M^*)
\right].
\end{aligned}
\]
Since the Kullback--Leibler divergence is always non-negative, we have
\[
D_{\mathrm{KL}}
\left[
P(\mathbf{X}|G^*,M^*)
\parallel
P(\mathbf{X}|G',M^*)
\right]
\geq 0,
\]
which proves that
\[
\mathcal{L}(G^*)\geq \mathcal{L}(G').
\]

To prove the strictness of the inequality, suppose that there exists a graph $G'\neq G^*$ such that $\mathcal{L}(G^*)=\mathcal{L}(G')$. This implies, $P(\mathbf{X}|G^*,M^*)=P(\mathbf{X}|G',M^*)$. Therefore, the two distinct graphs $G^*$ and $G'$ induce the same observational distribution under the known mechanism $M^*$. However, this contradicts the assumption that $M^*$ belongs to a strictly identifiable mechanism class, under which each graph corresponds to a unique induced distribution. Hence, no graph $G'\neq G^*$ can achieve the same expected log-likelihood as $G^*$.

Therefore, the equality holds if and only if $G'=G^*$.
\end{proof}
\subsection{Proof of Theorem \ref{thm:marginal_identifiability}}
\label{app:proof_Marginal_Likelihood}

To show the identifiability under the broad hypothesis space of causal mechanisms $\mathcal{M}$, we will need to first define the strict identifiability property from the set of causal prior mechanisms. Formally, we say the space of causal prior mechanisms is strictly identifiable if $P(\mathbf{X}|M^*,G^*)\ne P(X|M_{G'},G')$ holds for all alternative graphs $G' \neq G^*$, where $M^*$ and $G^*$ denote the true causal mechanism and true causal structure, respectively, and $M_{G'} = \arg\max_{M \in \mathcal{M}} \mathbb{E}_{\mathbf{X}}[\log P(\mathbf{X}|M,G')]$ represents the optimal mechanism choice under $G'$. This formulation serves as a direct extension from the traditional known causal mechanism setting, further suggesting that the entire space of causal prior mechanisms must possess the intrinsic property of mutual exclusivity. Under this structural foundation, the exact marginal likelihood uniquely isolates the true causal graph.

\begin{proof}
Let $\mathbf{X}^{(N)} =\{x_{1} ,x_{2} ,\dotsc ,x_{N} \}$ be a dataset of $N$ independent and identically distributed (i.i.d.) samples drawn from the true causal distribution $P^{*} (x)=P(x|M^{*} ,G^{*} )$. For any candidate graph $G$, the exact marginal likelihood of the dataset is given by the integral over the mechanism space $\mathcal{M}$:
\begin{equation}
P(\mathbf{X}^{(N)} |G)=\int _{\mathcal{M}} P(\mathbf{X}^{(N)} |M,G)P(M|G)dM=\int _{\mathcal{M}}\left(\prod _{i=1}^{N} P(x_{i} |M,G)\right) P(M|G)dM.
\end{equation}
We evaluate the scaled log marginal likelihood as $N\rightarrow \infty $. Following standard asymptotic Bayesian theory, the integral is asymptotically dominated by the maximum likelihood estimate over the parameter space, and we have:
\begin{equation}\label{eq:laplace_limit}
\lim _{N\rightarrow \infty }\frac{1}{N}\log P(\mathbf{X}^{(N)} |G)=\sup _{M\in \mathcal{M}}{}\mathbb{E}_{\mathbf{X}}[\log P(\mathbf{X} |M,G)] ,\ \ \text{a.s.}
\end{equation}
We now evaluate Eq.~\eqref{eq:laplace_limit} under two distinct structural scenarios. In case 1, the candidate graph is the true graph $G=G^{*}$. Since the true data-generating distribution is contained in the model class under $G^*$, the supremum is attained by any mechanism inducing $P(\mathbf X\mid M^*,G^*)$. By definition, we have
\begin{equation*}
\sup _{M\in \mathcal{M}}\mathbb{ E}_{\mathbf{X}}\left[\log P(\mathbf{X} |M,G^{*} )\right] =\mathbb{E}_{\mathbf{X}}\left[\log P(\mathbf{X} |M^{*} ,G^{*} )\right] .
\end{equation*}
In case 2, an alternative candidate graph $G=G'\neq G^{*}$. Let $M_{G'} \in \mathcal{M}$ be the optimal mechanism that the alternative graph topology can possibly choose to fit the data, and we have
\begin{equation}
\lim _{N\rightarrow \infty }\frac{1}{N}\log P(\mathbf{X}^{(N)} |G')=\mathbb{E}_{\mathbf{X}}[\log P(\mathbf{X}|M_{G'} ,G')].
\end{equation}
We then examine the difference between the two asymptotic limits:

\begin{equation}
\begin{aligned}
\Delta \mathcal{L} & =\mathbb{E}_{\mathbf{X}}\left[\log P(\mathbf{X} |M^{*} ,G^{*} )\right] -\mathbb{E}_{\mathbf{X}}[\log P(\mathbf{X}|M_{G'} ,G')]\\
 & =\mathbb{E}_{\mathbf{X}}\left[\log\frac{P(\mathbf{X} |M^{*} ,G^{*} )}{P(\mathbf{X} |M_{G'} ,G')}\right]\\
 & =D_{\mathrm{KL}}\bigl( P(\mathbf{X} |M^{*} ,G^{*} )\parallel P(\mathbf{X} |M_{G'} ,G')\bigr) .
\end{aligned}
\end{equation}
Since the mechanism class $\mathcal{M}$ is strictly identifiable, we have:
\begin{equation}
P(\mathbf{X} |M_{G'} ,G')\neq P(\mathbf{X} |M^{*} ,G^{*} ),
\end{equation}
guaranteeing that the KL divergence is strictly positive:
\begin{equation}
D_{\mathrm{KL}}\bigl( P(\mathbf{X}|M^{*} ,G^{*} )\parallel P(\mathbf{X}|M_{G'} ,G')\bigr)  >0.
\end{equation}
Consequently, we conclude that:
\begin{equation}
\lim_{N \to \infty} \frac{1}{N} \log P(\mathbf{X}^{(N)} | G^*) > \lim_{N \to \infty} \frac{1}{N} \log P(\mathbf{X}^{(N)} | G'), \quad \text{a.s.}, \quad \forall G' \neq G^*.
\end{equation}
This completes the proof that the exact marginal likelihood uniquely identifies the true causal graph $G^*$ almost surely as the sample size approaches infinity.
\end{proof}

\subsection{Proof of Corollary \ref{cor:elbo_identifiability}}

\begin{proof}
Let $\mathcal{J} (G;\mathbf{X}^{(N)} ):=\mathbb{E}_{Q^{*} (M\mid \mathbf{X}^{(N)} )}\bigl[\log P(G\mid M)\bigr]$ denotes the structural inference in CDFM. We prove that its maximizer converges almost surely to the true graph $G^{*}$ and $\displaystyle Q^{*}$ denotes the optimum variational posterior obtained by pretraining.

We begin with the connection between the structural inference and the joint likelihood $\displaystyle \log P(\mathbf{X}^{(N)} ,G)$, which can be expressed by,
\begin{equation*}
\log P(\mathbf{X}^{(N)} ,G)=\mathcal{L} (G,\mathbf{X}^{(N)} ;Q)+D_{\mathrm{KL}}\left( Q(M\mid \mathbf{X}^{(N)} )\parallel P(M\mid G,\mathbf{X}^{(N)} )\right) ,
\end{equation*}
where
\begin{equation*}
\mathcal{L} (G,\mathbf{X} ;Q)=\mathbb{E}_{Q(M\mid \mathbf{X} )}\bigl[\log P (\mathbf{X} \mid G,M)+\log P (G\mid M)\bigr] -D_{\mathrm{KL}}\bigl( Q(M\mid \mathbf{X} )\parallel P(M)\bigr)
\end{equation*}
is the fitted ELBO of CDFM given in Eq. \eqref{eq:elbo}:

Then, given that the foundation model has sufficient capacity to reach the optimum variational ELBO, we obtain 
\begin{equation}
\log P(\mathbf{X}^{(N)} ,G)=\mathcal{L} (G,\mathbf{X}^{(N)} ;Q^{*} ),
\end{equation}
where the optimal variational posterior $\displaystyle Q^{*} (M\mid \mathbf{X}^{(N)} )$ approaches $P(M\mid G,\mathbf{X}^{(N)} )$. Thus, since
\begin{equation}
\log P(\mathbf{X}^{(N)} ,G)=\log P(\mathbf{X}^{(N)} \mid G)+\log P(G),
\end{equation}
and under a uniform graph prior $\displaystyle P(G)$, we have
\begin{equation}
\arg\max_{G}\log P(\mathbf{X}^{(N)} ,G)=\arg\max_{G}\log P(G\mid \mathbf{X}^{(N)} )=\arg\max_{G}\mathcal{L} (G,\mathbf{X}^{(N)} ;Q^{*} ).
\end{equation}

Subsequently, we establish the connection between $\arg\max_{G} \mathcal{J} (G;\mathbf{X}^{(N)} )$ and $\displaystyle \arg\max_{G}\mathcal{L} (G,\mathbf{X}^{(N)} ;Q^{*} )$. First, by Theorem~\ref{thm:marginal_identifiability}, strict identifiability of the mechanism space implies that the exact marginal likelihood uniquely identifies the true graph:
\begin{equation}
\lim _{N\rightarrow \infty }\frac{1}{N}\log P(\mathbf{X}^{(N)} \mid G^{*} ) >\lim _{N\rightarrow \infty }\frac{1}{N}\log P(\mathbf{X}^{(N)} \mid G'),\ \ \forall G'\neq G^{*} ,\ \ \text{a.s.}
\end{equation}
Since the graph prior $P(G)$ is uniform, Bayes' rule gives
\begin{equation*}
P(G\mid \mathbf{X}^{(N)} )\propto P(\mathbf{X}^{(N)} \mid G).
\end{equation*}
Therefore, for any $G'\neq G^{*}$,
\begin{equation}\label{eq:marginal_likelihood_separation}
\begin{aligned}
\frac{P(G'\mid \mathbf{X}^{(N)} )}{P(G^{*} \mid \mathbf{X}^{(N)} )} & =\frac{P(\mathbf{X}^{(N)} \mid G')}{P (\mathbf{X}^{(N)} \mid G^{*} )} .
\end{aligned}
\end{equation}
Taking logarithms and dividing by $N$, Eq.~\eqref{eq:marginal_likelihood_separation} implies that this ratio converges to zero almost surely. Hence
\begin{equation}\label{eq:graph_posterior_concentration}
P (G^{*} \mid \mathbf{X}^{(N)} )\rightarrow 1,\ \ \ \ P (G'\mid \mathbf{X}^{(N)} )\rightarrow 0,\ \ \forall G'\neq G^{*} ,\ \ \text{a.s.}
\end{equation}
Thus, the posterior distribution over graphs converges almost surely to a Dirac measure concentrated at $G^{*}$. Then, we have
\begin{equation}
\begin{aligned}
P (M\mid \mathbf{X}^{(N)} ) & =\sum _{G\in \mathcal{G}} P(M\mid G,\mathbf{X}^{(N)} )P(G\mid \mathbf{X}^{(N)} )\\
 & =P(M\mid G^{*} ,\mathbf{X}^{(N)} )P(G^{*} \mid \mathbf{X}^{(N)} )+\sum _{G'\neq G^{*}} P(M\mid G',\mathbf{X}^{(N)} )P (G'\mid \mathbf{X}^{(N)} ).
\end{aligned}
\end{equation}
By Eq.~\eqref{eq:graph_posterior_concentration}, the first posterior weight converges to one and all remaining posterior weights vanish almost surely, and we obtain
\begin{equation}
P (M\mid G^{*} ,\mathbf{X}^{(N)} )\overset{\mathrm{a.s.}}{\sim } P (M\mid \mathbf{X}^{(N)} ),\ \ \ \ N\rightarrow \infty .
\end{equation}
Therefore, since the optimal variational posterior $\displaystyle Q^{*} (M\mid \mathbf{X}^{(N)} )$ approach to $P(M\mid G,\mathbf{X}^{(N)} )$, we have $\displaystyle Q^{*} (M\mid \mathbf{X}^{(N)} )\overset{\mathrm{a.s.}}{\sim } P (M\mid \mathbf{X}^{(N)} )$. This justifies using the amortized encoder $Q(M\mid \mathbf{X}^{(N)} )$ in place of the true graph-conditioned mechanism posterior asymptotically.

Asymptotically, the structural inference simplifies to $\mathcal{J} (G;\mathbf{X}^{(N)} )=\mathbb{E}_{P(M|\mathbf{X}^{(N)} )} [\log P(G|M)]$. By applying Jensen's inequality, we have:
\begin{equation*}
\mathcal{J} (G;\mathbf{X}^{(N)} )=\mathbb{E}_{P(M\mid \mathbf{X}^{(N)} )}\bigl[\log P(G\mid M)\bigr] \leq \log\mathbb{E}_{P(M\mid \mathbf{X}^{(N)} )}\bigl[ P(G\mid M)\bigr] =\log P(G\mid \mathbf{X}^{(N)} ).
\end{equation*}
For the true graph $G^*$, the posterior $P(G^* \mid \mathbf{X}^{(N)}) \to 1$, which implies $\log P(G^* \mid \mathbf{X}^{(N)}) \to 0$, and thus, $\mathcal{J} (G^*;\mathbf{X}^{(N)} ) \to 0$. Thus, for $G^*$, both sides of Jensen's inequality approach zero, making the bound asymptotically tight:
\begin{equation}
\mathcal{J} (G;\mathbf{X}^{(N)} )=\log P(G\mid \mathbf{X}^{(N)} ).
\end{equation}
Since the uniform prior ensures $\arg\max_{G}\log P(G\mid \mathbf{X}^{(N)} )=\arg\max_{G}\log P(\mathbf{X}^{(N)} \mid G)$, and Theorem~\ref{thm:marginal_identifiability} guarantees the marginal likelihood is strictly optimum at $G^{*}$, it follows that:
\begin{equation*}
\hat{G} =\arg\max_{G\in \mathcal{G}}\mathcal{J} (G;\mathbf{X}^{(N)} )=G^{*} ,\ \ \text{a.s. as } N\rightarrow \infty .
\end{equation*}
This completes the proof.
\end{proof}

\section{Synthetic Mechanism Abbreviations}\label{app:mechanism_names}

\renewcommand{\thetable}{A\arabic{table}}
\setcounter{table}{0}

\begin{table}[H]
\centering
\caption{Abbreviations of synthetic mechanism families used in Table~\ref{tab:mechanism_robustness}.}
\label{tab:mechanism_names}
\begin{tabular}{ll}
\toprule
Abbreviation & Description \\
\midrule
CAM & Causal additive model \\
CPT & Conditional probability table for discrete variables \\
Discrete ANM & Discrete additive noise model \\
Linear & Linear structural equation model \\
Linear-Hetero & Linear model with heteroscedastic noise \\
Linear-Ordinal & Linear model with ordinal or discretized observations \\
Meas. Error & Mechanism with additional measurement error \\
Physical & Physics-inspired nonlinear mechanism \\
Physical-Hetero & Physical mechanism with heteroscedastic noise \\
PNL & Post-nonlinear causal model \\
RFF & Random Fourier feature nonlinear mechanism \\
RFF-Hetero & RFF mechanism with heteroscedastic noise \\
RFF-Ordinal & RFF mechanism with ordinal or discretized observations \\
Rounded & Mechanism with rounded observations \\
Time-Lag & Time-lagged causal mechanism \\
\bottomrule
\end{tabular}
\end{table}

\end{document}